\documentclass{article}
\usepackage[margin=1in]{geometry}

\usepackage{times, microtype}
\usepackage{titlesec}
\usepackage{mathtools}
\usepackage{amsthm}
\usepackage{amssymb}
\usepackage{tikz}
\usepackage{hyperref}
\usepackage{enumitem}
\usepackage{bbm}
\usepackage[numbers]{natbib}
\usepackage{algorithm}
\usepackage{cleveref}
\usepackage{thmtools, thm-restate}

\title{Semi-bandit Optimization in the Dispersed Setting}
\author{%
Maria-Florina Balcan\\
Carnegie Mellon University\\
\texttt{ninamf@cs.cmu.edu}
\and
Travis Dick\\
University of Pennsylvania\\
\texttt{tbd@seas.upenn.edu}
\and
Wesley Pegden\\
Carnegie Mellon University\\
\texttt{wes@math.cmu.edu}
}

\newcommand{\comment}[1]{}

\newcommand{\brac}[1]{\left(#1\right)}

\newtheorem{theorem}{Theorem}
\newtheorem*{theorem*}{Theorem}
\newtheorem{corollary}[theorem]{Corollary}
\newtheorem{lemma}[theorem]{Lemma}
\theoremstyle{definition}
\newtheorem{definition}[theorem]{Definition}

\newtheorem*{conj*}{Conjecture}

\theoremstyle{plain}
\theoremstyle{remark}

\newcommand{\N}{\mathbb{N}}

\newcommand{\R}{\mathbb{R}}

\newcommand{\cF}{\mathcal{F}}

\newcommand{\ccc}{\mathcal{C}}
\newcommand{\ppp}{{\cal P}}

\newcommand{\ep}{\varepsilon}
\newcommand \reals {\mathbb{R}}
\newcommand \expect {\operatorname{\mathbb{E}}}
\newcommand \prob {\operatorname{Pr}}

\newcommand \configs {\mathcal{C}}
\newcommand \ind [1]{\mathbb{I}\{#1\}}
\newcommand \vol {\operatorname{Vol}}
\newcommand \scount [1]{\bigl|\{#1\}\bigr|}
\newcommand \argmin{\operatorname*{argmin}}

\newcommand \rmin {\rho_{\rm min}}
\newcommand \rmax {\rho_{\rm max}}

\newcommand \cA {\mathcal{A}}
\newcommand \cY {\mathcal{Y}}
\newcommand \dmin {\operatorname{d_{\rm min}}}
\newcommand \dmax {\operatorname{d_{\rm max}}}
\newcommand \drho {\operatorname{d_{\rho}}}
\newcommand \norm [1]{\Vert#1\Vert}
\newcommand \deltamin {\Delta_\text{min}}
\newcommand \deltamax {\Delta_\text{max}}

\renewcommand \paragraph [1] {\vspace{0.4em}\noindent\textbf{#1}}
\newcommand \itparagraph [1] {\vspace{0.4em}\noindent\textit{#1}}

\begin{document}

\maketitle

\begin{abstract}
  The goal of data-driven algorithm design is to obtain high-performing
  algorithms for specific application domains using machine learning and data.
  Across many fields in AI, science, and engineering, practitioners will often
  fix a family of parameterized algorithms and then optimize those parameters to
  obtain good performance on example instances from the application domain.  In
  the online setting, we must choose algorithm parameters for each instance as
  they arrive, and our goal is to be competitive with the best fixed algorithm
  in hindsight.

  There are two major challenges in online data-driven algorithm design. First,
  it can be computationally expensive to evaluate the loss functions that map
  algorithm parameters to performance, which often require the learner to run a
  combinatorial algorithm to measure its performance. Second, the losses can be
  extremely volatile and have sharp discontinuities. However, we show that in
  many applications, evaluating the loss function for one algorithm choice can
  sometimes reveal the loss for a range of similar algorithms, essentially for
  free. We develop online optimization algorithms capable of using this kind of
  extra information by working in the semi-bandit feedback setting. Our
  algorithms achieve regret bounds that are essentially as good as algorithms
  under full-information feedback and are significantly more computationally
  efficient. We apply our semi-bandit results to obtain the first provable
  guarantees for data-driven algorithm design for linkage-based clustering and
  we improve the best regret bounds for designing greedy knapsack algorithms.
\end{abstract}

\section{Introduction}
\label{sec:introduction}
\paragraph{Overview.} This paper concerns data-driven algorithm design for
combinatorial settings, which is an important area at the intersection
of machine learning and computing that has been long of interest to
the AI
community~\citep{Horvitz01:Bayesian,Xu08:Satzilla,Leyton09:Empirical,Demmel05:Self}.
However, until recently, most algorithm design procedures did not have
any provable guarantees on their performance, especially in the
realistic online scenario. The first general online data-driven
algorithm design procedures with regret bounds were given by
\citet{Balcan17:Dispersion}, who studied the problem under
full-information and bandit feedback regimes. We develop efficient
semi-bandit algorithms that achieve nearly the same regret as their
full information algorithms, while being as efficient as their bandit
algorithms.

The goal of data-driven algorithm design is to use machine learning
 and data to decide what algorithm to use from a large (typically
 parametrized) family of algorithms for a given problem domain. For
 example, we may want to decide which clustering algorithm to use from
 a large family of clustering procedures in order to obtain the
 highest quality results. We are concerned with the online setting,
 where at each round the \emph{learner} chooses an algorithm from the
 family and receives a new instance of the problem. The problem is
 characterized by a loss function that measures the performance of
 each algorithm in the family for the given instance, and the goal is
 to select algorithms so that the cumulative performance of the
 learner is nearly as good as the best algorithm in hindsight for that
 sequence of problems.

The major challenge in these settings is that it is potentially
computationally expensive for the learner to characterize the loss
function for each round, since each run of the algorithm reveals the
value of the loss function for just the selected parameters. Moreover,
for combinatorial problems, small differences between two algorithms
can lead to a cascade of changes in their behavior and significantly
change their performance. However, when the algorithm family is
parameterized, it can often be shown that the losses---though highly
nonconvex in general---are at least piecewise Lipschitz in the
algorithm parameters, so we can phrase the problem as online
optimization of piecewise Lipschitz functions.

Prior work on piecewise Lipschitz optimization was limited to two
extreme feedback regimes: Either the learner carries out a
computationally expensive process to obtain full-information feedback
(i.e., it observes the loss of every algorithm in the family on each
instance), or accepts suboptimal regret bounds to work in the bandit
feedback setting (i.e., it only observes the loss of one algorithm for
each instance). This creates a tradeoff between computational
efficiency and good regret bounds. However, many data-driven algorithm
design problems exhibit rich additional structure that is ignored by
these two approaches.  We show that, surprisingly, evaluating the loss
function for a single algorithm can sometimes reveal the loss for a
range of similar algorithms, essentially for free; in the context of
the loss function, we show that an entire Lipschitz region can often
be learned at once. This motivates us to define a new learning model,
which we call the semi-bandit feedback setting for learning piecewise
Lispchitz functions.  Our new results in this model achieve the best
of both worlds: we can efficiently obtain the necessary feedback while
also having regret bounds that are nearly as good as under
full-information.

We instantiate our results for data-driven algorithm design on two
combinatorial problems. These are machine learning problems where the
goal is to learn an optimal algorithm, rather than a prediction
rule. Our results for optimizing over a family of greedy knapsack
algorithms improve over the procedures of \citet{Balcan17:Dispersion,
Gupta17:PAC}, and \citet{Cohen-Addad17:Online} by simultaneously being
more efficient and having tighter regret bounds. We also provide the
first online data-driven algorithm design procedures for a rich family
of linkage based clustering algorithms introduced by
\citet{Balcan17:Learning} that interpolates between single and
complete linkage, which are algorithms that are widely used in
practice \citep{Awasthi14:ActiveClustering, Saeed03:Clustering,
White10:Phylo} and known to perform optimally in many settings
\citep{Awasthi12:PS,Balcan16:CenterPR,Balcan16:PR,Grosswendt15:CL}.
\citet{Balcan17:Learning} consider the data-driven algorithm design
problem for this family of algorithms in the batch setting, rather
than the online setting, where they model the application domain as a
distribution over problem instances, the goal is to find the algorithm
with the highest expected performance given an i.i.d. sample from the
distribution as training data.

\paragraph{Problem Setup.} We study the problem of online piecewise Lipschitz
optimization. The learning protocol is as follows: on each round $t$,
the learner chooses a parameter $\rho_t$ belonging to a
$d$-dimensional parameter space $\configs \subset \reals^d$, the
adversary chooses a piecewise Lipschitz loss function $\ell_t :
\configs \to [0,1]$, and the learner incurs a loss equal to
$\ell_t(\rho_t)$. A function $\ell_t : \configs \to [0,1]$ is
piecewise $L$-Lipschitz if we can partition the parameter space
$\configs$ into regions such that $\ell_t$ is $L$-Lipschitz when
restricted to each region. Many important instances of data-driven
algorithm design require optimizing piecewise Lipschitz functions,
including greedy combinatorial algorithms \citep{Gupta17:PAC},
clustering algorithms and SDP-rounding schemes
\citep{Balcan17:Learning}, branch and bound mixed integer program
solvers \citep{Balcan18:MIP}, initialization procedures for $k$-means
clustering \citep{Balcan18:kmeans}, and various auction design
problems \citep{Balcan18:MultiItem}. In these problems, the family of
algorithms is parameterized and each parameter $\rho \in \configs$
corresponds to one algorithm. We suppose that on each round $t$ there
is a partition $A^{(t)}_1, \dots, A^{(t)}_M$ of the parameter space
$\configs$, called the feedback system.  If the learner's parameter
$\rho_t$ belongs to the set $A^{(t)}_i$, then they observe both the
set $A^{(t)}_i$ as well as the loss $\ell_t(\rho)$ for every $\rho \in
A^{(t)}_i$. We consider the uninformed setting, where the learner does
not know the feedback system for round $t$ in advance of selecting a
parameter.  For simplicity, we consider oblivious adversaries that
choose their sequence of loss functions $\ell_1, \ell_2, \dots$
adversarially, but before the interaction with the learner begins. The
learner's goal is to minimize regret, which is the difference between
their total accumulated loss and that of the best parameter in
hindsight: $\sum_{t=1}^T \ell_t(\rho_t) - \min_{\rho \in \configs}
\sum_{t=1}^T \ell_t(\rho)$.

Throughout the paper, we use the notation $\tilde O(\cdot)$ to
optionally suppress all logarithmic terms and dependence on parameters
other than the time horizon $T$ and the dimension of the parameter
space $d$.

\paragraph{Main Results and Techniques.}

\itparagraph{Semi-bandit Regret Bounds in the Dispersed Setting.} It
is not always possible to achieve sub-linear regret for piecewise
Lipschitz loss functions \citep{Littlestone88:LQ,
BenDavid09:AgnosticOL, Rakhlin11:OL}.  \citet{Balcan17:Dispersion}
provide regret bounds in the full-information and bandit feedback
settings under a dispersion condition that roughly measures the number
of discontinuous functions in any ball of a given radius, and which is
satisfied for a diverse collection of combinatorial algorithm
configuration problems. In this paper, we introduce a related and more
general version of this condition that captures what is asymptotically
important for our regret bounds.
\begin{definition}\label{def:betadisp}
  The sequence of loss functions $\ell_1,\ell_2,\dots$ is
  \emph{$\beta$-point-dispersed} for the Lipschitz constant $L$ if for
  all $T$ and for all $\epsilon \geq T^{-\beta}$, we have that, in
  expectation, the maximum number of functions among
  $\ell_1,\dots,\ell_T$ that fail the $L$-Lipschitz condition for any
  pair of points at distance $\epsilon$ in $\configs$ is at most
  $\tilde O(\epsilon T)$. That is, for all $T$ and for all $\epsilon
  \geq T^{-\beta}$, we have
  $
    \expect \bigl[
      \max_{\rho,\rho'}\scount{t\in[T] \,:\, |\ell_t(\rho) - \ell_t(\rho')| > L\norm{\rho - \rho'}_2}
    \bigr]
    = \tilde O(\epsilon T).
  $
  where the max is taken over all $\rho, \rho' \in \configs : \norm{\rho -
  \rho'}_2 \leq \epsilon$.
\end{definition}
Note that the righthandside $\tilde O(\epsilon T)$ is roughly the
number $L$-Lipschitz failures one would expect across $T$ functions
for a pair of points at distance $\epsilon$ if Lipschitz failures are
distributed reasonably randomly, and their probability of occuring
between a pair of points at distance $\epsilon$ is roughly
proportional to $\epsilon$.  The definition of $\beta$-dispersion
measures how small $\epsilon$ can be while maintaining the correctness
of this rough bound for the loss functions $\ell_i$.

In our applications, the sequence of loss functions will be chosen by
a smoothed adversary, in the sense of
\citet{Spielman04:Smoothed}. Informally, the discontinuity locations
of the functions chosen by a smoothed adversary are randomly
perturbed. The expectation in \Cref{def:betadisp} is over this
randomness in the sequence of loss
functions. (\citet{Balcan17:Dispersion} also show examples where
sufficient randomness can arise from the algorithm itself, rather than
smoothness constraints on the adversary.) In all of our applications,
we prove $\beta$-dispersion with $\beta = 1/2$. We provide an
algorithm for online piecewise Lipschitz optimization under
semi-bandit feedback whose regret is characterized by the
$\beta$-dispersion parameter of the losses.  In \Cref{sec:online}, we
prove the following result:
\begin{theorem} \label{thm:betaRegret}
  Let $\configs \subset \reals^d$ be a bounded parameter space and
  $\ell_1, \ell_2, \dots : \configs \to [0,1]$ be piecewise Lipschitz
  functions that are $\beta$-point-dispersed. Running the continuous
  Exp3-SET algorithm (\Cref{alg:cExp3Set}) under semi-bandit feedback
  with an appropriate parameter $\lambda$ has expected regret bounded
  by
  $
  \expect\bigl[
    \sum_{t=1}^T \ell_t(\rho_t) - \ell_t(\rho^*)
  \bigr]
  \leq
  \tilde O\bigl(\sqrt{dT} + T^{1-\beta} \bigr).
  $
\end{theorem}
In comparison, the bandit-feedback algorithm of
\citet{Balcan17:Dispersion} has expected regret bounded by $\tilde O(d
T^{\frac{d+1}{d+2}} 3^d + T^{1-\beta})$.  Even in one-dimensional
problems, this bound is $\tilde O(T^{2/3} + T^{1-\beta})$, which is
worse than our results. Under different assumptions, the bandit
algorithm of \citet{Cohen-Addad17:Online} has $\tilde O(T^{2/3})$
regret for the special case of one-dimensional piecewise constant
functions.

\itparagraph{General Tools for Verifying Dispersion.} We also provide
general tools for proving that a sequence of piecewise Lipschitz
functions satisfies dispersion.  When the sequence $\ell_1, \ell_2,
\dots$ is random, we can usually directly bound the expected number of
loss functions that are not $L$-Lipschitz between any fixed pair of
points $\rho$ and $\rho'$ with $\norm{\rho - \rho'}_2 \leq \epsilon$
by $\tilde O(T\epsilon)$. However, this does not imply that the
functions are $\beta$-point-dispersed, since the expected number of
non-Lipschitz functions between the \emph{worst} pair of points at
distance $\epsilon$ will typically be larger than the expected number
for any fixed pair.  Building on uniform convergence from learning
theory~\citep{SBD14:MLBook}, we show that if each loss function has a
one-dimensional domain, at most $K$ discontinuities and any interval
of radius $\epsilon$ has at most $\tilde O(T\epsilon)$ non-Lipschitz
functions in expectation, then the expected number of non-Lipschitz
losses on the worst interval of length $\epsilon$ is at most $\tilde
O(T \epsilon + \sqrt{T\log(TK)})$. This implies that for all pairs of
points at distance $\epsilon$, at most $O(T \epsilon + \sqrt{T
\log(TK)})$ functions are non-Lipschitz between them and demonstrates
$\beta$-dispersion with $\beta = 1/2$. Our result gives an exponential
improvement in the dependence on $K$ compared to the results of
\citet{Balcan17:Dispersion}, who upper bound the expected number of
non-Lipschitz losses in the worst interval of length $\epsilon$ by
$\tilde O(TK \epsilon + K \sqrt{T\log(TK)})$.

\itparagraph{Semi-bandit Online Data-driven Algorithm Design.} In
\Cref{sec:applications}, we combine our general regret analysis from
\Cref{thm:betaRegret} together with application-specific dispersion
analysis to obtain practical data-driven algorithm design procedures
for linkage-based clustering and the knapsack problem. In both
applications, we show that the discontinuities of each loss function
are the roots of polynomials depending on the corresponding problem
instance, and that the roots are dispersed under mild smoothness
assumptions on the adversary. We obtain the first online data-driven
algorithm design procedures for linkage based clustering, and
algorithm design procedures for the knapsack problem with substantial
computational improvements over the prior work, while at the same time
achieving nearly the same regret bound.

\itparagraph{Explicit Comparison for Knapsack.} To highlight the
benefits of our new learning model and results applied to data-driven
algorithm design, we give an explicit comparison of the computational
complexity for obtaining different types of feedback and the
corresponding regret bounds for the family of greedy knapsack
algorithms introduced in \Cref{sec:knapsack}.  In each round of the
online game, the algorithm chooses a parameter $\rho$, a new knapsack
instance with $n$ items arrives, and our goal is for the total value
of items selected by the learner to be close to the total value of the
best fixed parameter $\rho$ in hindsight. We compare our results to
the best prior full-information and bandit feedback procedures.

\vspace{0.6em}
\begin{itemize}[nosep, leftmargin=*]
\item \noindent\textit{Full-information.}  \citet{Balcan17:Dispersion}
show that the exponentially weighted forecaster with full-information
feedback achieves a regret bound of $\tilde O(n^2\sqrt{T})$. Our
tighter analysis improves the bound to $\tilde O(\sqrt{T})$. Obtaining
full-information feedback has a total cost of $O(n^3 \log n)$ time per
round.
\item\noindent\textit{Bandit Feedback.} The discretization-based
bandit algorithm of \citet{Balcan17:Dispersion} has regret $\tilde
O(T^{2/3}n^2)$, but only requires $O(n \log n)$ time per round.
\item\noindent\textit{Semi-bandit Feedback.} In this paper we give an
algorithm whose regret is $\tilde O(n\sqrt{T})$ using semi-bandit
feedback obtainable in time $O(n \log n)$ per round.  Note that our
algorithm is as efficient as the bandit-feedback algorithm, yet its
regret is only larger by a factor of $n$
\end{itemize}

\paragraph{Related Work.} There is a rich literature on data-driven algorithm
design. Most prior work focuses on the statistical setting, where the
learner is given a large iid sample of problem instances from some
distribution, and the goal is to find the algorithm with the best
performance in expectation.  \citet{Gupta17:PAC} introduced this
formal setting and provide sample complexity results for several
families of greedy algorithms. \citet{Balcan17:Learning} consider
semidefinite rounding schemes for integer quadratic programs and
linkage based clustering algorithms, \citet{Balcan18:MIP} consider
learning the best branch and bound algorithms for mixed integer
programs, and \citet{Balcan18:kmeans} consider learning the best
initialization procedures for $k$-means
clustering. \citet{Aamand19:FrequencyEstimation} and
\citet{Hsu18:FrequencyEstimation} use learned algorithms for streaming
frequency estimation. \citet{Indyk19:LowRankApprox} study the problem
of using a learned sketching matrix to improve low-rank approximation
algorithms.  \citet{Dong19:NNS} use learned space partitions to
improve nearest neighbor search. In addition to these formal results,
this statistical setting has been the predominant model for
data-driven algorithm configuration in artificial intelligence
\citep{Rice76:AlgSelection}, combinatorial auctions
\citep{Leyton09:Auctions}, numerical linear algebra
\citep{Demmel05:Adapting}, vehicle routing
\citep{Caseau99:VehicleRouting}, and SAT solving
\citep{Xu08:Satzilla}.

Another related line of work focuses on the problem of choosing the
algorithm with the shortest running time over a distribution of
problem instances \citep{kleinberg2017efficiency,
Weisz18:LeapsAndBound, Weisz18:CapsAndRuns}.  This work makes minimal
assumptions about the algorithm family and instead designs procedures
that can avoid running every algorithm to completion, since this may
be very expensive. Our work, on the other hand, explores special
structure in algorithm families and can be used to optimize more
general performance measures in the online rather than stochastic
setting.

For online optimization of one-dimensional piecewise constant
functions, \citet{Cohen-Addad17:Online} provide full-information and
bandit online optimization procedures. \citet{Balcan17:Dispersion}
consider the more general setting of multi-dimensional piecewise
Lipschitz functions. They introduce a dispersion condition that
roughly measures how many functions are not Lipschitz in any ball, and
provide algorithms with dispersion-dependent full-information and
bandit regret bounds. They also verify that dispersion is satisfied
for a diverse collection of data-driven algorithm design problems.

Prior work on semi-bandit feedback has focused predominantly on
finite-armed bandits. Semi-bandit feedback was first considered for
online shortest path problems, where on each round the learner selects
a path through a graph and observes the length of the edges along that
path (but not for other edges) \citep{Gyorgy06:SPP,
Kale10:SPP}. \citet{Audibert14:Combinatorial} obtain minimax bounds
for a generalization to combinatorial bandits, where the learner's
action space is described by boolean vectors in $\{0,1\}^d$, the
losses are linear, and the on each round the learner observes the
entries of the loss vector corresponding to the non-zero entries in
their action.  \citet{Alon17:GraphFeedback} introduce the Exp3-SET
algorithm for semi-bandit feedback for finite-armed bandits. They
consider the graph-feedback setting introduced by
\citet{Mannor11:Semibandit}, where on each round $t$, there is a
feedback graph $G_t$ over the arms of the bandit and playing arm $i$
reveals the loss for arm $i$ and all arms adjacent in the graph
$G_t$. We extend the Exp3-SET algorithm to online optimization
problems where there are infinitely many arms and where the feedback
system on each round is a partition of the parameter space $\configs$.

\section{Semi-bandit Optimization of Piecewise Lipschitz Functions}
\label{sec:online}
In this section we provide an algorithm for online piecewise Lispchitz
optimization and analyze its regret under dispersion.  Our results are
for the following continuous semi-bandit setting.

\begin{definition} [Uninformed Semi-bandit Feedback.]
  An online optimization problem with loss functions
  $\ell_1, \ell_2, \dots$ has semi-bandit feedback if for each time
  $t$, there is partition $A^{(t)}_1, \dots, A^{(t)}_M$ of the
  parameter space $\configs$, called a feedback system, such that when
  the learner plays point $\rho_t \in A^{(t)}_i$, they observe the set
  $A^{(t)}_i$ and $\ell_t(\rho)$ for all $\rho \in A^{(t)}_i$. For any
  $\rho \in \configs$, we let $A^{(t)}(\rho)$ denote the feedback set
  that contains $\rho$.
\end{definition}

We analyze a continuous version of the Exp3-SET algorithm of
\citet{Alon17:GraphFeedback}. This algorithm uses importance weighting
to construct unbiased estimates of the complete loss function on each
round, which it passes as input to a continuous version of the
exponentially weighted forecaster. Pseudocode is given in
\Cref{alg:cExp3Set}. Unlike the Exp3 algorithm of \citet{Auer02:MAB},
the Exp3-SET algorithm and our continuous version do not include an
explicit exploration term (i.e., we do not mix the distribution $p_t$
with a uniform distribution over $\configs$).
\citet{Stoltz05:incomplete} was the first to show that mixing with the
uniform distribution is unnecessary for the Exp3 algorithm to have
optimal expected regret.

In \Cref{app:oneDimensionalImplementation}, we show how to implement
this algorithm with $O(\log T)$ per round time complexity for one
dimensional piecewise constant losses using the interval tree data
structure of \citet{Cohen-Addad17:Online}.

\begin{algorithm}
\noindent \textbf{Parameter:} Step size $\lambda \in [0,1]$
\begin{enumerate}[nosep, leftmargin=*]
\item Let $w_1(\rho) = 1$ for all $\rho \in \configs$
\item For $t = 1, \dots, T$
  \begin{enumerate}[nosep, leftmargin=*]
  \item Let $p_t(\rho) = \frac{w_t(\rho)}{W_t}$, where $W_t = \int_\configs w_t(\rho) \, d\rho$.
  \item Sample $\rho_t$ from $p_t$, play it, and observe feedback set $A^{(t)}(\rho)$ and losses $\ell_t(\rho)$ for all $\rho \in A_t$.
  \item Let $\hat \ell_t(\rho) = \frac{\ind{\rho \in A^{(t)}(\rho_t)}}{p_t(A^{(t)}(\rho_t))} \ell_t(\rho)$, where we define $p_t(A^{(t)}(\rho_t)) = \int_{A^{(t)}(\rho_t)} p_t(\rho) \, d\rho$.
  \item Let $w_{t+1}(\rho) = w_t(\rho) \exp(-\lambda \hat \ell_t(\rho))$ for all $\rho$.
  \end{enumerate}
\end{enumerate}
\caption{Continuous Exp3-SET}
\label{alg:cExp3Set}
\end{algorithm}

Given the learner's observations on round $t$, \Cref{alg:cExp3Set}
uses importance weighting to estimate the complete loss function by
$\hat \ell_t(\rho) = \frac{\ind{\rho \in
    A^{(t)}(\rho_t)}}{p_t(A^{(t)}(\rho_t))} \ell_t(\rho)$. The
estimate $\hat \ell_t(\rho)$ is only non-zero for parameters $\rho$
that belong to the feedback set observed by the algorithm at round
$t$.  The key property of $\hat \ell_t$ is that it is an unbiased
estimate of the true loss function conditioned on the history until
the beginning of round $t$. More formally, let
$\expect_t[\cdot] = \expect[\cdot | \rho_1, \dots, \rho_{t-1}, \ell_1,
\dots, \ell_t]$ denote the conditional expectation given the learner's
choices until round $t-1$ and the first $t$ loss functions. This
expectation is only over the randomness of the learner's choice of
$\rho_t$ at time $t$. For clarity, we also use the notation
$\expect_{<t}[\cdot]$ to denote the expectation of any random variable
that is a function of only $\rho_1, \dots, \rho_{t-1}$ and
$\ell_1, \dots, \ell_t$ so that for any random quantity $X$, we have
$\expect[X] = \expect_{<t}\bigl[\expect_t[X]\bigr]$. For any
$\rho \in \configs$ and $t$, a straight forward calculation shows that
$\expect_t[\hat \ell_t(\rho)] = \ell_t(\rho)$.

To simplify presentation, we assume that the sequence of loss
functions has an \emph{$r_0$-interior minimizer}: with probability
one, for all times $T$ there exists
$\rho^* \in \argmin_{\configs} \sum_{t=1}^T \ell_t(\rho)$ such that
$B(\rho^*, r_0) \subset \configs$.  We can usually modify a sequence
of loss functions to obtain an equivalent optimization problem that is
guaranteed to have an $r_0$-interior minimizer. In \Cref{app:online}
we discuss such a transformation that works whenever the parameter
space $\configs$ is convex (with no condition on the losses).

We bound the regret of \Cref{alg:cExp3Set} under a slightly more
precise version of $\beta$-point-dispersion which leads to more
precise bounds and broader applicability.

\begin{definition}\label{fdisp}
  The sequence of loss functions $\ell_1,\ell_2,\dots$ is
  \emph{$f$-point-dispersed} for the Lipschitz constant $L$ and
  dispersion function $f : \mathbb{N} \times [0,\infty) \to \reals$ if
  for all $T$ and for all $\epsilon > 0$, we have
  $
    \expect \bigl[
      \max_{\rho,\rho'}\scount{t\in[T] \,:\, |\ell_t(\rho) - \ell_t(\rho')| > L\norm{\rho - \rho'}_2}
    \bigr]
    \leq f(T,\epsilon).
  $
  where the max is taken over all
  $\rho, \rho' \in \configs : \norm{\rho - \rho'}_2 \leq \epsilon$.
\end{definition}

We can express both $\beta$-point-dispersion and $(w,k)$-dispersion
from \citet{Balcan17:Dispersion} in terms of $f$-point-dispersion. For
any $T \in \N$ and $\epsilon > 0$, let
$D(T,\epsilon) = \expect[\max_{\norm{\rho - \rho'}_2 \leq \epsilon}
\scount{1 \leq t \leq T \,:\, |\ell_t(\rho) - \ell_t(\rho')| \geq
  L\norm{\rho - \rho'}_2}$ be the expected number of non-Lipschitz
functions among $\ell_1, \dots, \ell_T$ across the worst pair of
points within distance at most $\epsilon$. If the loss functions are
$\beta$-point-dispersed, then we know that for all $T$ and
$\epsilon \geq T^{-\beta}$, we have
$D(T,\epsilon) = \tilde O(T\epsilon)$. Since $D(T,\epsilon)$ is a
non-decreasing function of the distance $\epsilon$, we are guaranteed
that for any $\epsilon < T^{-\beta}$ we have
$D(T,\epsilon) \leq D(T, T^{-\beta}) = \tilde O(T^{1-\beta})$. It
follows that the functions are also $f$-point-dispersed for
$f(T,\epsilon) = \tilde O(T\epsilon + T^{1-\beta})$. Similarly, the
functions are $(w,k)$-dispersed if every ball of radius $w$ in
$\configs$ has at most $k$ non-Lipschitz functions.  Since any pair of
points within distance $\epsilon$ are contained in a ball of radius
$\epsilon$, it follows that for $\epsilon \leq w$ we ahve
$D(T,\epsilon) \leq k$, but for $\epsilon > w$ we could have
$D(T,\epsilon)$ as large as $T$.  It follows that the functions are
$f$-point-dispersed where $f(T,\epsilon) = k$ for all $\epsilon < w$
and $f(T,\epsilon) = T$ otherwise.

We bound the regret of \Cref{alg:cExp3Set} in terms of the
$f$-point-dispersion of the losses. The proof is given in
\Cref{app:online}.

\begin{restatable}{theorem}{thmCExpSet} \label{thm:cExp3Set}
  Let $\configs \subset \reals^d$ be contained in a ball of radius $R$
  and $\ell_1, \ell_2, \dots : \configs \to [0,1]$ be piecewise
  $L$-Lipschitz functions that are $f$-point-dispersed with an
  $r_0$-interior minimizer.  Moreover, suppose the learner gets
  semi-bandit feedback and, on each round $t$, the feedback system
  $A^{(t)}_1, \dots, A^{(t)}_M$ has $M$ feedback sets.  For any
  $r \in (0,r_0]$, running \Cref{alg:cExp3Set} with
  $\lambda = \sqrt{d \log(R/r)/(TM)}$ satisfies the following regret
  bound:
  $
  \expect\bigl[
    \sum_{t=1}^T \ell_t(\rho_t) - \ell_t(\rho^*)
  \bigr]
  \leq
  O\bigl(\sqrt{dTM \log(R/r)} + f(T, r) + TLr\bigr).
  $
\end{restatable}

Our regret bound for $\beta$-dispersed losses given in
\Cref{thm:betaRegret} follows immediately from \Cref{thm:cExp3Set}.

Note that our results are also applicable in two closely related
settings: maximizing dispersed piecewise Lipschitz utility functions,
and the case when losses are bounded in $[0,H]$ for some known bound
$H$ instead of $[0,1]$. A discussion of the necessary transformations
can be found in \Cref{app:utilitiesAndBoundedLosses}.

\section{A Recipe for Verifying Dispersion}
\label{sec:dispersionTool}
In this section we illustrate a general recipe for proving dispersion
in data-driven algorithm design problems. We work in the framework of
smoothed analysis~\citep{Spielman04:Smoothed} and suppose that nature
injects a small amount of randomness into the problem instances chosen
by the adversary before the learner sees them. Our goal is to leverage
this framework to prove that the loss functions are dispersed.

At a high-level, a general strategy for proving dispersion in this
setting which has proved successful across a range of examples is to:
\begin{enumerate}
\item Bound the probability density of the random set of
discontinuities of the loss functions, to obtain a bound on the
typical rate of Lipschitz condition violations.
\item Use a VC-dimension based uniform convergence argument to
transform this typical rate into a bound on the dispersion of the loss
functions.
\end{enumerate}

In this section, we give general tools which can be used to accomplish
each of these steps in real-world problems.

For many combinatorial algorithm families, the loss function for a
given instance is piecewise $L$-Lipschitz on a partition of $\configs$
whose boundaries are defined by the roots of a collection of
polynomials.  In the smoothed analysis setting, the coefficients of
these polynomials have bounded probability density, and may (or may
not) be independent.  The following theorem translates this randomness
in the coefficients into a statement about the randomness of their
roots, making it easy to accomplish Step 1 in the strategy above.
\begin{restatable}{theorem}{thmPolyRoots} \label{thm:polyRoots}
  Consider a random degree $d$ polynomial $\phi(\rho)$ with leading
  coefficient 1 and subsequent coefficients which are real of absolute
  value at most $R$, whose joint density is at most $\kappa$. There is
  an absolute constant $K$ depending only on $d$ and $R$ such that
  every interval $I$ of length $\leq \epsilon$ satisfies
  $\prob(\text{$\phi$ has a root in $I$}) \leq \kappa \epsilon / K$.
\end{restatable}

(In \Cref{app:dispersionTool} we prove a generalization of
\Cref{thm:polyRoots} that allows for less structured coefficient
vectors.)

In the 1-dimensional setting (i.e., when optimizing a single-parameter
family of algorithms), \Cref{thm:polyRoots} often allows us to argue
that no interval of width $\epsilon$ contains any discontinuity from
each loss function with large probability.  In the multidimensional
setting, the sets of discontinuities of the $L$-Lipschitz loss
functions will often be algebraic curves (or in more than 2
dimensions, algebraic varieties) defined as the zero sets of
multivariate polynomials.  In this case, \Cref{thm:polyRoots} can
still be used to accomplish Step 1 of the dispersion strategy, by
showing that few zeros are likely to occur on any fixed
piecewise-linear path (on whose pieces the zero sets of the
multivariate polynomial is the zero set of a single-variable
polynomial).  In particular, this accomplishes Step 1 of the basic
strategy for proving dispersion.

For Step 2, we wish to transform our bound on the typical rate of
Lipschitz violations to a uniform bound on the worst number of
Lipschitz violations, over all pairs of points $\rho,\rho'$.  For
example, the following theorem accomplishes this in the 1-dimensional
case:

\begin{restatable}{theorem}{thmDispersionTool} \label{thm:dispersionTool}
  Let $\ell_1, \ell_2, \dots : \reals \to \reals$ be independent
  piecewise $L$-Lipschitz functions, each having at most $K$
  discontinuities. Let
  $
  D(T, \epsilon, \rho)
  = \scount{1 \leq t \leq T \,|\, \hbox{$\ell_t$ is not $L$-Lipschitz on $[\rho-\epsilon, \rho+\epsilon]$}}
  $
  be the number of functions in $\ell_1, \dots, \ell_T$ that are not
  $L$-Lipschitz on the ball $[\rho-\epsilon, \rho+\epsilon]$. Then we
  have
  $
  \expect[\max_{\rho \in \reals} D(T,\epsilon,\rho)]
  \leq
  \max_{\rho \in \reals}\expect[D(T, \epsilon, \rho)] + O(\sqrt{T \log(TK)}).
  $
\end{restatable}

To see the general utility of \Cref{thm:dispersionTool}, observe that
if in Step 1 we show that for all times $T$, radiuses $\epsilon > 0$
and any fixed interval $I$ of radius $\epsilon$, the expected number
of non-Lipschitz functions on interval $I$ is at most $\tilde
O(T\epsilon)$, then \Cref{thm:dispersionTool} guarantees that the
losses are $\frac{1}{2}$-dispersed.

To accomplish Step 2 in the case of higher dimensions with
discontinuities given as the 0 sets of (multivariate) polynomials, the
0-sets are now not finite sets but finite-degree algebraic curves (or
varieties).  To verify dispersion, we need a uniform-convergence bound
on the number of Lipschitz failures between the worst pair of points
$\rho,\rho'$ at distance $\leq \ep$, but the definition allows us to
bound the worst rate of discontinuties along any path between
$\rho,\rho'$ of our choice.  The following theorem bounds the VC
dimension of axis aligned segments against bounded-degree algebraic
curves, which will allow us to accomplish Step 2 by considering
piecewise axis-aligned paths between points $\rho$ and $\rho'$.
\begin{restatable}{theorem}{thmCurves} \label{thm:curves}
  There is a constant $K_d$ (e.g., $K_2\leq 11$) depending only on $d$
  such that axis-aligned line segments cannot shatter any collection
  of $K_d$ algebraic curves of degree at most $d$.
\end{restatable}
The proof, which appears in the appendix, makes repeated use of
Bezout's theorem which bounds the number of intersection points of
algebraic curves in terms of their degrees.  In particular, a family
of $k$ algebraic curves will always a $\mathrm{poly}(k)$-bounded
number of intersection points and local extrema, which, one can show,
makes it impossible to label the exponentially-many subsets of such
curves with axis-aligned segments.

\Cref{thm:curves} allows us now to obtain a 2-dimensional analog of
\Cref{thm:dispersionTool} as follows, giving an implementation of Step
2 in this setting.

\begin{restatable}{theorem}{thmCurvesTool} \label{thm:curvesTool}
  Let $\ell_1, \ell_2, \dots : \reals^2 \to \reals$ be independent
  piecewise $L$-Lipschitz functions, each having a set of
  discontinuities specified by a collection of $K$ algebraic curves of
  bounded degree. Let $\mathcal{L}$ denote the set of axis-aligned
  line-segments in $\reals^2$. For each $s \in \mathcal{L}$, define
  $D(T, s) = \scount{1 \leq t \leq T \,:\, \text{$\ell_t$ has a
  discontinuity along $s$}}$. Then we have $\expect[\sup_{s \in
  \mathcal{L}} D(T,s)] \leq \sup_{s \in \mathcal{L}} \expect[D(T,s)] +
  O(\sqrt{T \log(TK)})$.
\end{restatable}

\section{Online Data-driven Algorithm Design with Semi-bandit Feedback}
\label{sec:applications}
In this section we apply our semi-bandit optimization results to online
data-driven algorithm design for two rich parameterized families of algorithms.
For both families, we show how to obtain semi-bandit feedback by running a
single algorithm from the family. We also analyze dispersion for these problems
under the assumption that the adversary is smoothed. In both cases, we obtain
$\tilde O(\sqrt{T})$ regret bounds in the semi-bandit feedback setting. Finally,
in \Cref{app:binSearch} we show how to use binary search to obtain semi-bandit
feedback for a large class single-parameter algorithm families.

\paragraph{Smoothed adversaries.} We consider adversaries that are smoothed in
the sense of \citet{Spielman04:Smoothed}, where their decisions are corrupted by
small random perturbations. Formally, we say that a parameter chosen by the
adversary is $\kappa$-smooth if it is a random variable whose density is bounded
by $\kappa$. After the adversary chooses the density for each smoothed
parameter, nature samples each parameter value independently from their
corresponding distributions. Small values of $\kappa$ correspond to larger
random perturbations of the problem parameters, while in the limit as $\kappa
\to \infty$, the adversary is able to choose the parameters deterministically.
In each application, we will specify which problem parameters are smoothed,
together with the bound $\kappa$ on their density. For simplicity, we assume
that all $\kappa$-smooth random variables are independent (i.e., the corruption
of the adversary's choices is not correlated across variables), though many of
our results can be exteneded to allow for some correlation between the
parameters of each instance.

\subsection{Greedy Algorithms for Knapsack}
\label{sec:knapsack}

First, we consider selecting the best algorithm from a parameterized family of a
greedy algorithms for the knapsack problem. An instance of the knapsack problem
consists of $n$ items, where item $i$ has a value $v_i$ and a size $s_i$, and a
knapsack capacity $C$. Our goal is to find the most valuable subset of items
whose total size does not exceed $C$. \citet{Gupta17:PAC} propose using the
following parameterized family of greedy knapsack algorithms: for a given
parameter $\rho \in [0,R]$, set the score of item $i$ to be $\sigma_\rho(i) =
v_i / s_i^\rho$. Then, in decreasing order of score, add each item to the
knapsack if there is enough capacity left. This algorithm runs in time $O(n \log
n)$. In our analysis, we assume that the adversary's item values are
$\kappa$-smooth.

First, we show how to obtain semi-bandit feedback for this family of greedy
knapsack algorithms by running a single algorithm in the family. Pseudocode is
given in \Cref{alg:knapsack}.

\begin{algorithm}
\noindent \textbf{Input:} Parameter $\rho \geq 0$, item values $v_1, \dots, v_n$, item sizes $s_1, \dots, s_n$, knapsack capacity $C \geq 0$.
\begin{enumerate}[nosep, leftmargin=*]
\item Let $\pi : [n] \to [n]$ be the item permutation such that $\sigma_\rho(\pi(1)) \geq \dots \geq \sigma_\rho(\pi(n))$.
\item Initialize $S \leftarrow \emptyset$.
\item For $i = 1, \dots, n$: if $s_{\pi(i)} \leq C$ then add $\pi(i)$ to $S$ and set $C \leftarrow C - s_{\pi(i)}$.
\item For $i = 1, \dots, n-1$: let $c_i \leftarrow \frac{\log(v_{\pi(i)} / v_{\pi(i+1)})}{ \log(s_{\pi(i)} / s_{\pi(i+1)})}$.
\item Let $\rmin \leftarrow \max \{ c_i \,|\, c_i \leq \rho \}$.
\item Let $\rmax \leftarrow \min \{ c_i \,|\, c_i > \rho\}$.
\item Return $S$ and interval $A = (\rmin, \rmax)$.
\end{enumerate}
\caption{Semi-bandit Knapsack}
\label{alg:knapsack}
\end{algorithm}

\begin{restatable}{lemma}{lemKnapsackFeedack}\label{lem:knapsackFeedback}
  Consider a knapsack instance with capacity $C$ and $n$ items with values $v_1,
  \dots, v_n$ and sizes $s_1, \dots, s_n$. \Cref{alg:knapsack} runs in time $O(n
  \log n)$. Moreover, there is a feedback system $A_1, \dots, A_M$ partitioning
  $\configs$ into $M = O(n^2)$ intervals such that set of items output by the
  algorithm is constant for $\rho \in A_i$. When run with parameter $\rho$, in
  addition to the item set $S$, the algorithm outputs the interval $A_i$
  containing $\rho$.
\end{restatable}
\begin{proof}[Proof sketch]
  The items selected by the algorithm only depend on the item ordering $\pi$.
  Steps 4 and 5 compute the largest parameter interval containing $\rho$ with
  the same item ordering as $\rho$, and therefore the items output by the
  algorithm is constant on this interval. Based on the work of
  \citet{Gupta17:PAC}, we know there are at most $O(n^2)$ such intervals.
\end{proof}

In contrast to \Cref{alg:knapsack}, the most direct approach to obtaining
full-information feedback for this family of knapsack algorithms is to first
compute a set of $O(n^2)$ critical parameter values arising from all pairs of
points and to run the algorithm once for each cell in the corresponding
partition, taking $O(n^3 \log n)$ time.

Next, we provide a dispersion analysis for selecting the parameter $\rho \in
[0,R]$ in order to maximize the value of items selected. We assume that each
instance has the same capacity $C$, item sizes are in $[1,C]$, and the item
values are in $[0,1]$ and $\kappa$-smooth. The corresponding loss function is
$\ell(\rho) = C - \sum_{i \in S_\rho} v_i \in [0,C]$, where $S_\rho$ is the set
of items selected by \Cref{alg:knapsack} when run with parameter $\rho$.

\begin{restatable}{lemma}{lemKnapsackDispersion} \label{thm:knapsackDispersion}
  Consider an adversary choosing knapsack instances with a fixed knapsack
  capacity $C$ where the $t^{\rm th}$ instance has item sizes $s^{(t)}_1, \dots,
  s^{(t)}_n \in [1,C]$, and $\kappa$-smooth item values $v^{(t)}_1, \dots,
  v^{(t)}_n \in [0,1]$. The loss functions $\ell_1, \ell_2, \dots$ defined above
  are piecewise constant, $f$-dispersed for $f(T,\epsilon) = T\epsilon n^2
  \kappa^2 \ln(C) + O(\sqrt{T \log(Tn)})$, and $\beta$-dispersed for $\beta =
  1/2$.
\end{restatable}
\begin{proof}
  Let $c^{(t)}_{ij} = \log(v^{(t)}_i / v^{(t)}_j) / \log(s^{(t)}_i / s^{(t)}_j)$
  be the critical parameter value such that at $\rho = c^{(t)}_{ij}$, items $i$
  and $j$ swap their relative order in the $t^{\rm th}$ instance.
  \citet{Balcan17:Dispersion} show that each critical value $c^{(t)}_{ij}$ is
  random and has a density function bounded by $\kappa^2 \ln(C)/2$. It follows
  that for any interval $I$ of radius $\epsilon$, the expected total number of
  critical values $c^{(t)}_{ij}$ summed over all pairs of items and $t = 1,
  \dots, T$ is at most $T \epsilon n^2 \kappa^2 \ln(C)$. This is also an upper
  bound on the expected number of loss functions in $\ell_1, \dots, \ell_T$ that
  are not constant on $I$. Applying \Cref{thm:dispersionTool}, it follows that
  the functions are $f$-dispersed for $f(T,\epsilon) = T \epsilon n^2 \kappa^2
  \ln(C) + O(\sqrt{T \log(Tn)}) = \tilde O(T\epsilon + \sqrt{T})$, which implies
  $\beta$-dispersion with $\beta = 1/2$.
\end{proof}

Running \Cref{alg:cExp3Set} using the semi-bandit feedback returned by
\Cref{alg:knapsack}, we obtain the following.

\begin{corollary}
  Under the same conditions as \Cref{thm:knapsackDispersion}, using
  \Cref{alg:cExp3Set} to tune the parameter $\rho \in [0,R]$ of
  \Cref{alg:knapsack} under semi-bandit feedback has expected regret
  bounded by $O(C n \sqrt{T \log(RTn\kappa\log(C))})$.
\end{corollary}

The full-information regret bound obtained by
\citet{Balcan17:Dispersion} is $\tilde O(C n^2 \sqrt{T})$, which is
worse than our semi-bandit bound (but can be improved to $\tilde O(C
\sqrt{T})$ using our tighter dispersion analysis).

\subsection{Interpolating between Single and Complete Linkage Clustering}
\label{sec:clusteringMerge}

Next, we consider a rich family of linkage-based clustering algorithms
introduced by \citet{Balcan17:Learning} that interpolates between the
classic single and complete linkage procedures. Clustering instances
are described by a matrix $D = (d_{ij}) \in \reals^{n \times n}$
giving the pairwise distances between a collection of $n$ data points
and the goal is to organize the points into a hierarchy or cluster
tree. We provide the first dispersion analysis and online
configuration procedures for this class of algorithms. We assume that
each distance $d_{ij}$ is $\kappa$-smooth.

The algorithm family we consider, called $\rho$-linkage, is family of
agglomerative clustering algorithms with a single parameter $\rho \in
[0,1]$.  These algorithms take as input a distance matrix $D \in
\reals^{n \times n}$ with entries $d_{ij}$ and the parameter value
$\rho \in [0,1]$ and output a cluster tree, which is a binary tree
where each node corresponds to a cluster in the data. The leaves of
the tree are the individual data points, while the root node
corresponds to the entire dataset. The children of each node subdivide
that cluster into two subclusters. The $\rho$-linkage algorithm starts
with each point belonging to its own cluster. Then, it repeatedly
merges the closest pair of clusters according the distance defined by
$\drho(A, B) = (1-\rho) \dmin(A, B) + \rho \dmax(A,B)$, where $A$ and
$B$ are clusters (i.e., subsets of $[n]$), $\dmin(A,B) = \min_{a \in
A, b \in B} d_{ab}$ and $\dmax(A,B) = \max_{a \in A, b \in B}
d_{ab}$. When there is only a single cluster remaining, the algorithm
outputs the constructed cluster tree.

For any pair of candidate cluster merges $(C_1, C_2)$ and $(C'_1,
C'_2)$, where $C_1, C_2, C'_1$ and $C'_2$ are clusters, there is a
critical parameter value $c$ such that $d_\rho(C_1, C_2) =
d_\rho(C'_1, C'_2)$ only when $\rho = c$. To simplify notation in the
rest of this section, we let $c(C_1, C_2, C'_1, C'_2) = \deltamin /
(\deltamin - \deltamax)$, where $\deltamin = \dmin(C'_1, C'_2) -
\dmin(C_1, C_2)$ and $\deltamax = \dmax(C'_1, C'_2) - \dmax(C_1,
C_2)$.

First, we show how to obtain semi-bandit feedback for this family of
linkage algorithms by running a single algorithm in the family. Our
modified algorithm maintains an interval $(\rmin, \rmax)$ with the
invariant that at any iteration, for all parameters $\rho' \in (\rmin,
\rmax)$, the algorithm would make the same merges that have been made
so far. Pseudocode for this procedure is given in
\Cref{alg:rhoLinkageBookkeeping}

\begin{algorithm}
\noindent \textbf{Input:} Parameter $\rho \in [0,1]$, distance matrix $D \in \reals^{n \times n}$.
\begin{enumerate}[nosep, leftmargin=*]
\item Let $S \leftarrow \{ \operatorname{Leaf}(i) \hbox{ for $i \in [n]$}\}$.
\item Let $\rmin \leftarrow 0$ and $\rmax \leftarrow 1$.
\item While $|S| > 1$:
  \begin{enumerate}[nosep, leftmargin=*]
    \item Let $(C_1,C_2) = \argmin_{C_1, C_2 \in S} d_\rho(C_1,C_2)$.
    \item For each pair $(C'_1, C'_2) \neq (C_1, C_2)$ in $S$
    \begin{enumerate}[nosep, leftmargin=*]
      \item Let $c' \leftarrow c(C_1, C_2, C'_1, C'_2)$.
      \item If $c' > \rho$ then set $\rmax \leftarrow \min(\rmax, c')$, otherwise set $\rmin \leftarrow \max(\rmin, c')$.
    \end{enumerate}
    \item Remove $C_1$ and $C_2$ and add $\operatorname{Node}(C_1, C_2)$ to $S$.
  \end{enumerate}
\item Return the only element $T$ of $S$ and $A = [\rmin, \rmax]$.
\end{enumerate}
\caption{Semi-bandit $\rho$-Linkage}
\label{alg:rhoLinkageBookkeeping}
\end{algorithm}

\begin{restatable}{lemma}{lemRhoLinkageBookkeeping} \label{lem:rhoLinkageBookkeeping}
  Consider a clustering instance with distance matrix $D \in \reals^{n
  \times n}$. \Cref{alg:rhoLinkageBookkeeping} runs in time
  $O(n^3)$. Moreover, there is a feedback system $A_1, \dots, A_M$
  partitioning $[0,1]$ into $M = O(n^8)$ intervals such that the
  cluster tree output by the algorithm is constant for $\rho \in
  A_i$. When run with parameter $\rho$, in addition to the cluster
  tree $T$, the algorithm outputs the interval $A_i$ containing
  $\rho$.
\end{restatable}
\begin{proof}[Proof sketch]
  On each iteration, we compute the critical parameter values where
  the pair of clusters chosen in step (a) of
  \Cref{alg:rhoLinkageBookkeeping} would change.  All parameters in
  the largest interval containing $\rho$ and no critical parameter
  values from any iterations will result in exactly the same
  clustering. \citet{Balcan17:Learning} showed that each clustering
  instance has at most $O(n^8)$ discontinuities, which bounds the
  number of feedback sets obtained in this way.
\end{proof}

Similarly to the knapsack example, the most direct approach for
obtaining full-information feedback is to first calculate a set of
$O(n^8)$ critical parameter values arising from all $O(n^8)$ subsets
of $8$ points and to run $\rho$-linkage once for each interval in the
corresponding partition. By using a priority queue to maintain the
distances between clusters, it is possible to implement $\rho$-linkage
in $O(n^2 \log n)$ time. This leads to a total running time of
$O(n^{10} \log n)$---much higher than the $O(n^3)$ running time in
\Cref{lem:rhoLinkageBookkeeping}. Note that using a priority queue in
\Cref{alg:rhoLinkageBookkeeping} does not reduce the running time to
$O(n^2 \log n)$, since updating the interval $(\rmin, \rmax)$ requires
a linear pass through all $O(n^2)$ pairs of clusters, so finding the
closest pair faster does not reduce the running time.

Next, we provide a dispersion analysis for selecting the parameter
$\rho$ of \Cref{alg:rhoLinkageBookkeeping} when the clustering
instances are chosen by a smoothed adversary. In particular, we
suppose that on each round the adversary chooses a distance matrix
$D^{(t)}$ where each distance $d^{(t)}_{ij}$ is $\kappa$-smooth and
takes values in $[0,B]$. The quantity $B / (1/\kappa) = B\kappa$
roughly captures the scale of the perturbations relative to the true
distances. Our analysis leads to regret that depends on $B\kappa$ only
logarithmically and give good bounds even for exponentially small
perturbations.

Fix any loss function $g : \reals^{n \times n} \times
\textsc{ClusterTrees} \to [0,1]$, where $g(D, T)$ measures the cost of
cluster tree $T$ for distance matrix $D$. For example, $g(D,T)$ could
be the $k$-means cost of the best $k$-pruning of the tree $T$ or the
distance to a ground-truth target clustering.  We study the loss
functions given by $\ell_t(\rho) = g(D^{(t)}, \cA(D^{(t)}; \rho))$,
where $\cA(D; \rho)$ denotes the output cluster tree of
\Cref{alg:rhoLinkageBookkeeping} run on distance matrix $D$ with
parameter $\rho$.

\begin{restatable}{lemma}{lemLinkageDispersion} \label{lem:linkageDispersion}
  Consider an adversary choosing clustering instances where the
  $t^{\rm th}$ instance has symmetric distance matrix $D^{(t)} \in
  [0,B]^{n \times n}$ and for all $i \leq j$, $d^{(t)}_{ij}$ is
  $\kappa$-smooth. The losses $\ell_1, \ell_2, \dots$ defined above
  are piecewise constant, $f$-dispersed for $f(T,\epsilon) = 32 T
  \epsilon n^8 \kappa^2 M^2 + O(\sqrt{T \log(Tn)})$ and
  $\beta$-dispersed for $\beta = 1/2$.
\end{restatable}
\vspace{-0.1cm}
\begin{proof}[Proof sketch.]
  In the proof of \Cref{lem:rhoLinkageBookkeeping}, we showed that for
  each time $t$, there are $O(n^8)$ critical parameter values
  partitioning $\configs$ into regions so that the algorithm output is
  constant on each region. Since the loss $\ell_t$ only depends on
  $\rho$ through the algorithm output, $\ell_t$ is also piecewise
  constant with at most $O(n^8)$ pieces.

  Moreover, we argued that every discontinuity of $\ell_t$ occurs at a
  critical parameter value of the form $c = (d^{(t)}_{rr'} -
  d^{(t)}_{ii'})/(d^{(t)}_{jj'} - d^{(t)}_{ii'} + d^{(t)}_{rr'} -
  d^{(t)}_{ss'})$ where $i, i', j, j', r, r', s, s'$ are $8$ point
  indices.  Similarly to the knapsack example, we show that each
  critical parameter value is random and has a density function
  bounded by $16(\kappa B)^2$. From this, it follows that for any
  interval $I$ of radius $\epsilon$, summing over all times $t = 1,
  \dots, T$ and all subsets of $8$ points, we have that the expected
  total number of critical values that land in interval $I$ is at most
  $32 T\epsilon (\kappa B)^2$. This also bounds the expected number of
  functions $\ell_1, \dots, \ell_T$ that are not constant on $I$. By
  \Cref{thm:dispersionTool}, the functions are $f$-dispersed for
  $f(T,\epsilon) = 32 T \epsilon (\kappa B)^2 + \sqrt{T \log(Tn)} =
  \tilde O(T\epsilon + \sqrt{T})$, also implying
  $\frac{1}{2}$-dispersion.

  There are several cases when bounding the density of the critical
  value $c$, depending on whether any of the 4 distances correspond to
  the same entry in the distance matrix $D$. We give the argument for
  the case when all 4 distances are distinct entries and therefore
  independent. The remaining cases are similar and considered in
  \Cref{app:applications}. Let $X = d_{rr'} - d_{ii'}$ and $Y =
  d_{jj'} - d_{ss'}$ so that $c = X / (X+Y)$. The variables $X$ and
  $Y$ are independent. Since $X$ and $Y$ are each the sum of
  $\kappa$-smooth random variables, \Cref{lem:kappaSum} implies that
  they are each have $\kappa$-bounded densities. Using the fact that
  $|X+Y| \leq 2B$, applying \Cref{lem:kappaRatio1} implies that the
  ratio $c = X/(X+Y)$ has a $16(\kappa B)^2$ bounded density, as
  required.
\end{proof}

Running \Cref{alg:cExp3Set} using the semi-bandit feedback returned by
\Cref{alg:rhoLinkageBookkeeping}, we obtain the following:

\begin{corollary}
  Under the same conditions as \Cref{lem:linkageDispersion}, using
  \Cref{alg:cExp3Set} to tune the parameter $\rho \in [0,1]$ of
  \Cref{alg:rhoLinkageBookkeeping} under semi-bandit feedback has
  expected regret bounded by
  $
  O(n^4 \sqrt{T \log(T n \kappa B)}).
  $
\end{corollary}
In \Cref{sec:metricAndLinkage} we show how to extend these results to
apply to the case of also learning a metric in addition to
interpolating between single and complete linkage.

\section{Conclusion}
\label{sec:conclusion}
In this work, we provide the first online optimization algorithm for piecewise
Lipschitz functions under semi-bandit feedback with regret bounds that depend on
the dispersion of the loss functions. We also give general tools for verifying
dispersion in applications with exponentially tighter bounds than prior work.
Finally, we apply our results to two data-driven algorithm design problems. We
obtain the first online data-driven algorithm design procedure for a family of
linkage-based clustering algorithms, and an online data-driven algorithm design
procedure for a greedy family of knapsack algorithms that is more efficient and
has better regret bounds than prior work. A cornerstone of our results is that,
for many data-driven algorithm design problems, semi-bandit feedback can be
obtained as efficiently as bandit-feedback and is sufficient for our algorithms
to achieve nearly the same regret bounds as under full-information feedback. Our
results largely mitigate the tradeoff between computational efficiency and good
regret bounds suffered by prior approaches, making online data-driven algorithm
design practical.

\vspace{1cm}
\paragraph{Acknowledgements.}
This work was supported in part by NSF grants CCF-1535967,
IIS-1618714, IIS-1901403, CCF-1910321, SES-1919453, DMS1700365, the
Defense Advanced Research Projects Agency under cooperative agreement
HR00112020003, an AWS Machine Learning Research Award, an Amazon
Research Award, a Microsoft Research Faculty Fellowship, a Bloomberg
Research Grant, and by the generosity of Eric and Wendy Schmidt by
recommendation of the Schmidt Futures program.

\bibliographystyle{plainnat}
\bibliography{references}


\clearpage

\appendix

\section{Appendix for Online Optimization (\Cref{sec:online})}
\label{app:online}
\paragraph{Problem Transformations to Obtain $r_0$-interior Minimizers.} Recall
that a sequence of loss functions $\ell_1, \ell_2, \dots$ has an
$r_0$-interior minimizer if with probability 1, for all times $T$ we
have that there exists $\rho^* \in \argmin_{\rho \in \configs}
\sum_{t=1}^T \ell_t(\rho)$ such that $B(\rho^*, r_0) \subset
\configs$. We can usually modify a sequence of loss functions to
obtain an equivalent optimization problem that is guaranteed to have
an $r_0$-interior minimizer. For example, when the parameter space
$\configs$ is convex (e.g., a cube in $\reals^d$, which covers most
algorithm configuration applications), we can apply the following
transformation: define an enlarged parameter space $\configs' =
\bigcup_{\rho \in \configs} B(\rho, r_0)$ and a modified sequence of
loss functions $\ell'_t : \configs \to [0,1]$ given by $\ell'_t(\rho')
= \ell_t(\Pi_\configs(\rho'))$, where $\Pi_\configs$ denotes the
Euclidean projection onto $\configs$. Using the fact that projections
onto convex sets are contractions, it follows that the sequence
$\ell'_1, \ell'_2, \dots$ is also $L$-Lispchitz and
$f$-dispersed. Moreover, it has an $r_0$-interior minimizer and any
sequence of parameters $\rho'_1, \rho'_2, \dots \in \configs'$ can be
converted into $\rho_1, \rho_2, \dots \in \configs$ by taking $\rho_t
= \Pi_\configs(\rho'_t)$. This guarantees that $\ell_t(\rho_t) =
\ell'_t(\rho'_t)$ for all $t$. In particular, an algorithm with low
regret playing against $\ell'_1, \ell'_2, \dots$ can be converted into
one that plays against $\ell_1, \ell_2, \dots$ with an identical
regret bound.  The cost of this transformation is that it increases
the diameter of the parameter space $\configs$ by $2r$. Our regret
bounds have logarithmic dependence on the diameter of $\configs$.

\thmCExpSet*
\begin{proof}[Proof of \Cref{thm:cExp3Set}]
  For the majority of the proof we consider an arbitrary deterministic
  sequence of piecewise Lipschitz loss functions $\ell_1, \dots,
  \ell_T$ with an $r_0$-interior minimizer. We will only suppose they
  are $f$-point-dispersed in the final steps of the proof.

  Following the proof of the Exp3-Set algorithm of
  \citet{Alon17:GraphFeedback}, we will upper and lower bound the
  quantity $\expect[\log(W_{T+1}/W_1)]$. Our upper bound will be in
  terms of the learner's total expected loss, while the lower bound
  will be in terms of the expected total loss of the optimal parameter
  in hindsight. Dispersion plays a crucial role in the lower bound,
  since it allows us to guarantee that a set of parameters with
  non-trivial volume has nearly optimal total loss. Combining these
  bounds and then finally taking the expectation of the bound for a
  sequence of losses $\ell_1, \dots, \ell_T$ that are $f$-dispersed
  will give the final bound.

  \vspace{1em}\noindent\textit{Upper Bound.} Consider the ratio of
  consecutive normalizing constants $W_{t+1}/W_t$. Using the
  definition of $w_{t+1}$ and $p_t$, we have
  \begin{align*}
    \frac{W_{t+1}}{W_t}
    &= \int_\configs \frac{w_{t}(\rho)}{W_t} \exp(-\lambda \hat \ell_t(\rho)) \, d\rho \\
    &= \int_\configs p_t(\rho) \exp(-\lambda \hat \ell_t(\rho)) \, d\rho.
  \end{align*}
  Next, using that $e^{-z} \leq 1 - z + z^2/2$ for all $z \geq 0$, we have
  \begin{align*}
    \frac{W_{t+1}}{W_t}
    &\leq \int_\configs p_t(\rho) \left(
      1 - \lambda \hat \ell_t(\rho) + \frac{\lambda^2}{2} \hat \ell_t(\rho)
    \right) \, d\rho \\
    &= 1 - \lambda \int_\configs p_t(\rho) \hat \ell_t(\rho) \, d\rho
        + \frac{\lambda^2}{2} \int_\configs p_t(\rho) \hat \ell_t(\rho)^2 \, d\rho.
  \end{align*}
  Using the fact that $1-z \leq \exp(-z)$ for all $z \geq 0$ and
  taking the product over $t = 1, \dots, T$, we have
  $
    \frac{W_{T+1}}{W_1}
    \leq
    \exp\bigl(
      -\lambda \sum_{t=1}^T \int_\configs p_t(\rho) \hat \ell_t(\rho) \, d\rho
      +\frac{\lambda^2}{2} \sum_{t=1}^T \int_\configs p_t(\rho) \hat \ell_t(\rho)^2 \, d\rho
    \bigr).
  $
  Taking logs, we have
  \begin{equation}
    \begin{split}
    \log(\frac{W_{T+1}}{W_1})
    \leq
      &-\lambda \sum_{t=1}^T \int_\configs p_t(\rho) \hat \ell_t(\rho) \, d\rho \\
      &+\frac{\lambda^2}{2} \sum_{t=1}^T \int_\configs p_t(\rho) \hat \ell_t(\rho)^2 \, d\rho.
    \end{split}
    \label{eq:upperBoundWithoutExpect}
  \end{equation}

  Next, we will take the expectation of the above bound to simplify
  the two integrals. Recall that for each time $t$, we let $A^{(t)}_1,
  \dots, A^{(t)}_M$ be the feedback system and for any $\rho \in
  \configs$ and let $A^{(t)}(\rho)$ denote the set $A^{(t)}_i$ such
  that $\rho \in A^{(t)}_i$. Recall that the importance-weighted
  losses $\hat \ell_t$ were constructed to ensure that for any time
  $t$ and any fixed $\rho \in \configs$, we have $\expect_t[\hat
  \ell_t(\rho)] = \ell_t(\rho)$.
  Therefore,
  \begin{align*}
  \expect\left[\int_\configs p_t(\rho) \hat \ell_t(\rho) \, d\rho\right]
  &=\expect_{<t}\left[\expect_t\left[\int_\configs p_t(\rho) \hat \ell_t(\rho) \, d\rho\right]\right] \\
  &=\expect_{<t}\left[\int_\configs p_t(\rho) \ell_t(\rho) \, d\rho \right].
  \end{align*}
  The integral in the final expectation is the definition of
  $\expect_t[\ell_t(\rho_t)]$, which gives $\expect\left[\int_\configs p_t(\rho)
  \hat \ell_t(\rho) \, d\rho\right] = \expect_{<t}[\expect_t[\ell_t(\rho_t)]] =
  \expect[\ell_t(\rho_t)]$. Therefore, we have
  \begin{equation}
    \expect\left[ \sum_{t=1}^T \int_\configs p_t(\rho) \hat \ell_t(\rho) \, d\rho \right]
    = \expect \left[ \sum_{t=1}^T \ell_t(\rho_t) \right],
    \label{eq:totalLoss}
  \end{equation}
  which is the total expected loss of the algorithm on the first $T$ rounds.

  Now we turn to simplifying the expectation of the second integral in
  \eqref{eq:upperBoundWithoutExpect}. For any $\rho \in \configs$ and
  any time $t$, we have
  \[
    \expect_t[\hat \ell_t(\rho)^2]
    = \int_\configs p_t(\rho') \left(\frac{\ind{\rho \in A^{(t)}(\rho')}}{p_t(A^{(t)}(\rho'))} \ell_t(\rho) \right)^2 \, d\rho.
  \]
  Using the fact that $\rho \in A^{(t)}(\rho')$ if and only if $\rho'
  \in A^{(t)}(\rho)$, we can upper bound the integral as follows:
  \begin{align*}
  \int_\configs p_t(\rho') &\left(\frac{\ind{\rho \in A^{(t)}(\rho')}}{p_t(A^{(t)}(\rho'))} \ell_t(\rho) \right)^2 \, d\rho \\
  &= \left(\frac{\ell_t(\rho)}{p_t(A^{(t)}(\rho))}  \right)^2 \cdot \int_{A^{(t)}(\rho)} p_t(\rho') \, d\rho \\
  &= \frac{\ell_t(\rho)^2}{p_t(A^{(t)}(\rho))} \\
  &\leq \frac{1}{p_t(A^{(t)}(\rho))}.
  \end{align*}
  This implies that
  \begin{align*}
  \expect \left[\int_\configs p_t(\rho) \hat \ell_t(\rho)^2\, d\rho\right]
  &=\expect_{<t} \left[ \expect_t \left[\int_\configs p_t(\rho) \hat \ell_t(\rho)^2\, d\rho\right] \right] \\
  &\leq \expect\left[\int_\configs p_t(\rho) \frac{1}{p_t(A^{(t)}(\rho))} \, d\rho\right]
  \end{align*}
  Finally, we evaluate the integral by writing it as the sum of
  integrals over the feedback sets $A^{(t)}_1, \dots, A^{(t)}_M$,
  which is justified since these sets partition $\configs$. In
  particular, we have
  \begin{align*}
  \int_\configs p_t(\rho) \frac{1}{p_t(A^{(t)}(\rho))} \, d\rho
  &= \sum_{i=1}^M \frac{1}{p_t(A^{(t)}_i)} \cdot \int_{A^{(t)}_i} p_t(\rho) \, d\rho \\
  &= M.
  \end{align*}
  Putting it together, we have
  \begin{equation}
    \expect \left[\sum_{t=1}^T \int_\configs p_t(\rho) \hat \ell_t(\rho)^2\, d\rho \right]
    \leq TM.
    \label{eq:totalLossSecondMoment}
  \end{equation}
  Taking the expectation of \eqref{eq:upperBoundWithoutExpect} and
  using the calculations given by \eqref{eq:totalLoss} and
  \eqref{eq:totalLossSecondMoment}, we have
  \[
    \expect\left[\log \frac{W_{T+1}}{W_1} \right]
    \leq -\lambda  \expect\left[ \sum_{t=1}^T\ell_t(\rho_t)\right] + \frac{\lambda^2}{2} TM.
  \]

  \vspace{1em}\noindent\textit{Lower Bound.} Next, let $\rho^* \in
  \argmin_{\rho \in \configs} \sum_{t=1}^T \ell_t(\rho)$ be such that
  $B(\rho^*, r_0) \subset \configs$ and fix any radius $r \leq
  r_0$. Using the fact that $W_1 = \int_{\configs} 1 \, d\rho =
  \vol(\configs)$ and the weights $w_{T+1}(\rho)$ are positive, we
  have
  \begin{align*}
    \frac{W_{T+1}}{W_1}
    &= \frac{1}{\vol(\configs)} \int_\configs w_{T+1}(\rho) \, d\rho \\
    &\geq \frac{1}{\vol(\configs)} \int_{B(\rho^*, r)} \exp\left(-\lambda \sum_{t=1}^T \hat \ell_t(\rho) \right) \, d\rho.
  \end{align*}
  Taking the log of this bounds gives
  $
    \log \frac{W_{T+1}}{W_1}
    \geq \log \frac{1}{\vol(\configs)} + \log\left( \int_{B(\rho^*, r)} \exp\left(-\lambda \sum_{t=1}^T \hat \ell_t(\rho) \right) \, d\rho \right).
  $
  At this point it is tempting to apply dispersion to lower bound the
  term $\exp\bigl(-\lambda \sum_t \hat \ell_t(\rho) \bigr)$ in terms
  of $\exp\bigl(-\lambda \sum_t \hat \ell_t(\rho^*) \bigr)$. In
  particular, if at each time $t$ the feedback system $A^{(t)}_1,
  \dots, A^{(t)}_M$ corresponds to the piecewise Lispchitz
  partitioning of $\configs$ for the loss function $\ell_t$, then the
  estimated loss function $\hat \ell_t$ has a subset of the
  discontinuities of $\ell_t$. In this case, the estimated losses
  $\hat \ell_1, \hat \ell_2, \dots$ are also $f$-dispersed for the
  same function $f$ as the true losses. However, when the feedback
  system at around $t$ does not match the piecewise Lipschitz
  partition, we would require a separate dispersion analysis for the
  sequence of estimated losses $\hat \ell_1, \hat \ell_2, \dots$. The
  more serious challenge, however, is that the importance weight
  $1/p_t(A^{(t)}(\rho_t))$ in the definition of $\hat \ell_t$ causes
  it to take values in the range $[0, 1/p_t(A^{(t)}(\rho_t))]$, which
  is much larger than the true losses which take values in
  $[0,1]$. Moreover, the Lipschitz parameter of the estimated loss
  $\ell_t$ is $L' = L/p_t(A^{(t)}(\rho_t))$.  This larger loss range
  and Lipschitz constant lead to a worse final regret bound. Instead,
  we defer applying dispersion until after taking expectations so that
  we can use the dispersion properties of the true losses $\ell_1,
  \ell_2, \dots$ directly.

  Towards this end, we first use Jensen's inequality to simplify the
  above bound. Let $h : \configs \to \reals$ be any function and $S
  \subset \configs$ be any subset of the parameter space. Then, using
  the fact that $\log$ is concave, we can apply Jensen's inequality to
  obtain the following bound:
  \begin{align*}
  \log &\left(\int_S \exp(h(\rho)) \, d\rho \right) \\
  &= \log \left(\vol(S) \int_{S} \frac{1}{\vol(S)} \exp(h(\rho)) \, d\rho \right) \\
  &= \log(\vol(S)) + \log \left(\int_{S} \frac{1}{\vol(S)} \exp(h(\rho)) \, d\rho \right) \\
  &\geq \log(\vol(S)) + \int_{S} \frac{1}{\vol(S)} \log(\exp(h(\rho))) \, d\rho \\
  &= \log(\vol(S)) + \int_{S} \frac{1}{\vol(S)} h(\rho) \, d\rho,
  \end{align*}
  Applying this inequality to our lower bound on $\log \frac{W_{T+1}}{W_1}$ with
  $h(\rho) = -\lambda \sum_{t=1}^T \hat \ell_t(\rho)$ and $S = B(\rho^*, r)$
  gives
  $
  \log \frac{W_{T+1}}{W_1}
    \geq \log \frac{\vol(B(\rho^*,r))}{\vol(\configs)} - \lambda \int_{B(\rho^*, r)} \frac{1}{\vol(B(\rho^*,r))} \sum_{t=1}^T \hat \ell_t(\rho) \, d\rho.
  $
  Next, since $\configs$ is contained in a ball of radius $R$ and the
  volume of a ball of radius $R$ in $\reals^d$ is proportional to
  $R^d$, it follows that the volume ratio is at least
  $(r/R)^d$. Taking expectations, we have
  $
  \expect\left[\log \frac{W_{T+1}}{W_1}\right]
  \geq d \log \frac{r}{R} - \lambda \int_{B(\rho^*, r)} \frac{1}{\vol(B(\rho^*,r))} \sum_{t=1}^T \ell_t(\rho) \, d\rho,
  $
  where we used the fact that for any fixed $\rho \in \configs$, we have
  $\expect[\hat \ell_t(\rho)] = \expect_{<t}[\expect_t[\hat \ell_t(\rho)]] =
  \ell_t(\rho)$.
  Finally, we will upper bound the sum of losses $\sum_{t=1}^T \ell_t(\rho)$ for
  points in the ball $B(\rho^*, r)$ in terms of the number of non-Lipschitz
  functions on the worst pair of points within distance $r$ in $\configs$. After
  taking expectations, this quantity will be bounded using $f$-point-dispersion.
  For any pair of points $\rho, \rho' \in \configs$, define
  $
  D_r = \sup_{\rho, \rho' \in \configs : \norm{\rho - \rho'}_2 \leq r} \scount{1 \leq t \leq T \,:\, |\ell_t(\rho) - \ell_t(\rho')| > L \norm{\rho - \rho'}_2}
  $
  to be the number of functions among $\ell_1, \dots, \ell_T$ that violate the
  $L$-Lipschitz condition between the worst pair of points $\rho$ and $\rho'$
  within distance $r$. Then, for any $\rho \in B(\rho^*, r)$ we have that
  \[
  \sum_{t=1}^T \ell_t(\rho)
  \leq \sum_{t=1}^T \ell_t(\rho) + TLr + D_r,
  \]
  since at most $D_r$ of the losses violate the $L$-Lipschitz
  condition between $\rho$ and $\rho^*$ and the remaining loss
  functions can change value by at most 1 between $\rho$ and $\rho^*$.
  Substituting into our bound gives
  \[
  \expect\left[\log \frac{W_{T+1}}{W_1}\right]
  \geq d \log \frac{r}{R} - \lambda \sum_{t=1}^T \ell_t(\rho^*) - TLr - D_r.
  \]

  \vspace{1em}\noindent\textit{Combined bound.} Combining the upper and lower
  bounds and rearranging, we have
  \begin{equation}
    \sum_{t=1}^T \expect[\ell_t(\rho_t)] - \ell_t(\rho^*)
    \leq \frac{\lambda}{2} TM + \frac{d}{\lambda} \log \frac{R}{r} + TLr + D_r.
    \label{eq:nonDispersedRegret}
  \end{equation}

  Finally, now suppose that the functions $\ell_1, \dots, \ell_T$ are
  a random sequence that satisfy $f$-point-dispersion. Taking the
  expectation of both sides of \eqref{eq:nonDispersedRegret} with
  respect to the loss sequence, we have that
  \begin{align*}
  \expect&\left[\sum_{t=1}^T \ell_t(\rho_t) - \ell_t(\rho^*) \right] \\
  &\leq \frac{\lambda}{2} TM + \frac{d}{\lambda} \log \frac{R}{r} + TLr + \expect[D_r] \\
  &\leq \frac{\lambda}{2} TM + \frac{d}{\lambda} \log \frac{R}{r} + TLr + f(T,r),
  \end{align*}
  where the final inequality follows from the definition of
  $f$-segment-disperison. The specific bounds given in the theorem statement
  follow by substituting the chosen value of $\lambda$.
\end{proof}

\subsection{Optimizing Utilities and $H$-Bounded Losses}
\label{app:utilitiesAndBoundedLosses}

We note that the regret bound obtained in \Cref{thm:cExp3Set} for
\Cref{alg:cExp3Set} can also be used to obtain similar results in two
closely related settings. First, if we instead have piecewise
Lipschitz utility functions $u_1, u_2, \dots : \configs \to [0,1]$ and
our goal is to maximize utility rather than minimize loss, we can
transform this into a loss-minimization problem by minimizing the
losses given by $\ell_t(\rho) = 1 - u_t(\rho)$. This transformation
preserves the regret of any algorithm, the feedback system at each
round, and the piecewise Lipschitz and dispersion properties of the
functions. Similarly, if the losses take values in $[0,H]$ for some
known maximum loss $H$, instead of $[0,1]$, the learner can preprocess
the losses to fall in $[0,1]$ by dividing them by $H$. The rescaled
functions take values in $[0,1]$ and have Lipschitz constant $L' =
L/H$. Then expected regret of \Cref{alg:cExp3Set} with respect to the
unscaled loss functions is $O(H\sqrt{TMd \log(R/r)} + Hf(T,r) + TLr)$.

\begin{restatable}{lemma}{lemUtilityToLoss}
  Let $u_1, u_2, \dots : \configs \to [0,H]$ be a sequence of utility
  functions that are each piecewise $L$-Lipschitz and
  $f$-dispersed. Define a corresponding sequence of losses $\ell_1,
  \ell_2, \dots : \configs \to [0,H]$ given by $\ell_t(\rho) = H -
  u_t(\rho)$. The functions $\ell_1, \ell_2, \dots$ are also piecewise
  $L$-Lispchitz and $f$-dispersed.
\end{restatable}
\begin{proof}
  First, consider any time $t$. Since $u_t : \configs \to [0,H]$ is
  piecewise $L$ Lipschitz, by definition we know that there is a
  partition $\configs_1, \dots, \configs_M$ of $\configs$ so that for
  each $i \in [M]$ and any $\rho, \rho' \in \configs_i$, we have
  $|u_t(\rho) - u_t(\rho')| \leq L \cdot \norm{\rho - \rho'}_2$. We
  will argue that the loss function $\ell_t$ is also piecewise
  $L$-lispschitz on the same partition. Fix any $i \in [M]$ and any
  pair of points $\rho, \rho' \in \configs_i$. Then we have that
  \begin{align*}
    |\ell_t(\rho) - \ell_t(\rho')|
    &= \bigl|(H - u_t(\rho)) - (H - u_t(\rho'))\bigr| \\
    &= |u_t(\rho') - u_t(\rho)| \\
    &\leq L \cdot \norm{\rho - \rho'}_2,
  \end{align*}
  where the last inequality follows from the fact that $u_t$ is
  $L$-Lipschitz restricted to $\configs_i$. It follows that $\ell_t$
  is also piecewise $L$-Lipschitz and has the same piecewise Lipschitz
  partition. This holds for all times $t$.

  Next, we argue that whenever the utility functions $u_1, u_2, \dots$
  are $f$-dispersed, so are the loss functions $\ell_1, \ell_2,
  \dots$. For any time horizon $T$, radius $\epsilon > 0$, and
  parameter $\rho \in \configs$, define
  \[
  D_u(T, \epsilon, \rho)
  = \scount{1 \leq t \leq T \,:\ \text{$u_t$ is not $L$-Lipschitz on $B(\rho, \epsilon)$}}
  \]
  and
  \[
  D_\ell(T, \epsilon, \rho)
  = \scount{1 \leq t \leq T \,:\ \text{$\ell_t$ is not $L$-Lipschitz on $B(\rho, \epsilon)$}}.
  \]
  Following an identical argument as in the first part, with
  probability 1, whenever $u_t$ is $L$-Lipschitz on the ball $B(\rho,
  \epsilon)$, so is the function $\ell_t$. From this, it follows that
  $D_u(T, \epsilon, \rho) = D_\ell(T, \epsilon, \rho)$ for all $T \in
  \mathbb{N}$, $\epsilon > 0$, and $\rho \in \configs$. Finally, since
  the functions $u_1, u_2, \dots$ were $f$-dispersed, we have that for
  all $T \in \mathbb{N}$ and all radiuses $\epsilon > 0$, we have
  \[
  \expect[ \max_{\rho \in \configs} D_u(T, \epsilon, \rho) ] \leq f(T,\epsilon).
  \]
  It follows that
  \[
  \expect[ \max_{\rho \in \configs} D_\ell(T, \epsilon, \rho) ]
  = \expect[ \max_{\rho \in \configs} D_u(T, \epsilon, \rho) ]
  \leq f(T,\epsilon),
  \]
  and the loss functions $\ell_1, \ell_2, \dots$ are also $f$-dispersed.
\end{proof}

\subsection{Efficient Implementations via Interval Trees}
\label{app:oneDimensionalImplementation}

In this section we show how to use the modified interval tree data
structure of \citet{Cohen-Addad17:Online} to implement the continuous
Exp3-SET algorithm efficiently for one-dimensional problems with
piecewise constant loss functions.  In particular, the per-round cost
of updating the algorithm weights and sampling from them at time $t$
is only $O(\log(t))$, while a direct implementation has running time
given by $O(t)$ instead. We also show how to use interval trees to
implement the Exp3 algorithm on a set of $N$ arms with per-round
running time that is $O(\log N)$, which implies that a
discretization-based algorithm in the bandit setting can be
efficiently implemented even in high dimensions.

\paragraph{Interval Tree Summary.} \citet{Cohen-Addad17:Online} introduce a modified
interval tree data structure used for representing piecewise constant
functions mapping $\reals$ to $\reals$. Their data structure
represents the function as a balanced tree with one leaf corresponding
to each constant interval of the function. It supports two main
operations called \textsc{Draw} and \textsc{Update}:
\begin{itemize}
\item The \textsc{Draw} procedure returns a sample from the density
function that is proportional to the represented piecewise constant
function $f$.
\item The \textsc{Update} procedure takes an interval $[a,b)$ and an
update $u$, and updates the represented piecewise function by
multiplying the function values in $[a,b)$ by $u$. That is, if the
represented function was originally $f : \reals \to \reals$, after
executing \textsc{Update} with interval $[a,b)$ and update $u$, the
resulting function is
\[
  f'(x) = \begin{cases}
    f(x) & \text{if $x \not \in [a,b)$} \\
    u\cdot f(x) & \text{if $x \in [a,b)$}.
  \end{cases}
\]
\end{itemize}
\citet{Cohen-Addad17:Online} show that the operations \textsc{Draw}
and \textsc{Update} can be implemented in $O(\log P)$ time, where $P$
is the number of constant pieces in the represented function. The data
structure also makes it possible to implement a third procedure
\textsc{Integrate} in $O(\log P)$ time, which takes an interval
$[a,b)$ and returns the integral of the represented represented
function on the interval $[a,b)$.

\paragraph{Exp3-Set for Piecewise Constant One Dimensional Problems.} First, we
show how to efficiently implement \Cref{alg:cExp3Set} efficiently for
one-dimensional optimization problems with piecewise constant loss
functions. We simply use the interval tree datastructure of
\citet{Cohen-Addad17:Online} to represent the weight function at each
round. Pseudocode is given in \Cref{alg:IntervalExp3Set}.

\begin{lemma}
  Consider an online optimization problem with loss functions $\ell_1,
  \ell_2 : [0,1] \to [0,1]$ that are piecewise constant. Moreover,
  suppose that on each round $t$, the loss $\ell_t$ is constant on
  each of the feedback sets $A^{(t)}_i$. For such a problem,
  \Cref{alg:IntervalExp3Set} is equivalent to \Cref{alg:cExp3Set}. The
  overhead of sampling and updating the weights on round $t$ takes
  $O(\log t)$ time.
\end{lemma}
\begin{proof}
  On each round we run $\textsc{Update}$ once to update the interval
  tree. This at most increases the number of constant intervals in the
  weights by 2, since the only constant intervals that might get split
  are the two containing the end points of the feedback set
  $A_t$. Therefore, on round $t$, the weight function is piecewise
  constant with at most $O(2t)$ intervals. It follows that the
  sampling, integration, and update operations all take $O(\log t)$
  time, giving a total per-round cost of $O(\log t)$.
\end{proof}

\begin{algorithm}
\noindent \textbf{Parameter:} Step size $\lambda \in [0,1]$
\begin{enumerate}[nosep, leftmargin=*]
\item Initialize $W$ to be the interval tree representing $w(\rho) = \ind{\rho \in [0,1]}$.
\item For $t = 1, \dots, T$
  \begin{enumerate}[nosep, leftmargin=*]
  \item Let $\rho_t \leftarrow \textsc{Draw}(W)$ and play $\rho_t$.
  \item Observe feedback interval $A_t = A^{(t)}(\rho)$ and loss $\ell_t(\rho_t)$.
  \item Let $\hat \ell_t \leftarrow \frac{\ell_t(\rho_t)}{p_t(A_t)}$, where $p_t(A_t) \leftarrow \frac{\textsc{Integrate}(W, A_t)}{\textsc{Integrate}(W, [0,1])}$.
  \item Call $\textsc{Update}(W, A_t, \hat \ell_t)$.
  \end{enumerate}
\end{enumerate}
\caption{Continuous Exp3-SET for Piecwise Constant One Dimensional Problems}
\label{alg:IntervalExp3Set}
\end{algorithm}

\section{Appendix for Dispersion Recipe (\Cref{sec:dispersionTool})}
\label{app:dispersionTool}
We begin by proving \Cref{thm:polyRoots}. In fact, we prove a more
general version, given below:

\begin{theorem}
  Consider real vector $\mathbf{a}=(a_1,\dots,a_k)^T$ $(k\leq d+1$,
  $|a_i|\leq R)$ \[\phi(x)=\alpha_dx^d+\alpha_{d-1}
  x^{d-1}+\dots+\alpha_0\] where the coefficient vector
  $(\alpha_{d-1},\dots,\alpha_{0})^T$ is $f(\mathbf{a})$ for some
  affine transformation $f:\R^k\to \R^{d+1}$ which has the property
  that for any real number $r$, there is some input to $f$ for which
  $r$ is not a root of $\phi(x)$.  Then there is an absolute constant
  $K$ depending only on $d$, $R$, and $f$ such that if the vector
  $\mathbf{a}$ is chosen randomly from a distribution of density
  bounded by $\kappa$, and there is some interval $I\subseteq \R$ of
  length $\leq \ep$ such that \[ \Pr(\phi \mbox{ has a root in
  $I$})\geq p, \] then we have that \[ \kappa \ep/p\geq K.  \]
\end{theorem}

\begin{proof}
We consider the polynomial
\[
\phi(x)=\alpha_dx^d+\alpha_{d-1}x^{d-1}+\dots+\alpha_1x^1+\alpha_0,
\]
where, in particular
\begin{equation}\label{alphaioperator}
  |\alpha_i|\leq R_f:=|f|_\infty R
\end{equation}
for all $i$, where $|f|_\infty$ denotes the $\ell^\infty$ operator norm of $f$.

Moreover, as the coefficient vector $(\alpha_{d},\dots,\alpha_0)$
belongs to the $f$-image of the cube $[-R,R]^k$, the set $S_A$ of
possible coefficient vectors has volume \[ V_A:=\sqrt{\det{A^T
A}}(2R)^k, \] where here $A$ denotes the linear part of $A$.

For a root $\varrho$ of $\phi$, we can factor $\phi(x)$ as
$\phi(x)=(x-\varrho)\psi(x)$ where we define
$\psi(x)=\brac{\beta_{d-1}x^{d-1}+\beta_{d-2}x^{d-2}+\dots+\beta_1x^1+\beta_0}$
and
\begin{equation}
  \label{etatonu}
  \alpha_j=\beta_{j-1}-\varrho\cdot \beta_{j}\quad\mbox{for each}\quad 1\leq j\leq d,
\end{equation}
and $\alpha_0=\varrho\cdot \beta_0$.  In particular, the vector $(\alpha_d,\alpha_{d-1},\dots,\alpha_0)^T$ is the product $B^{\varrho}\cdot \mathbf{b}$ where $\mathbf{b}$ is the vector $(\beta_{d-1},\beta_{d-2},\dots,\beta_{0})^T$ and $B^{\varrho}$ is the $(d+1)\times d$ matrix
\[
B^{\varrho}_{ij}=\begin{cases}
1 & j=i  \\
-\varrho & j=i+1\\
0 & otherwise.
\end{cases}
\]
In particular, allowing arbitrary $\beta_{d-1},\dots,\beta_0\in \R$,
the product $B^\varrho\cdot b$ defines a $d$-dimensional subspace
$H_{B,\varrho}$ of $\R^{d+1}$; this is precisely the subspace of
coefficient vectors $(\alpha_d,\dots,\alpha_0)$ for which $\rho$ is a
root of $\phi$.  By assumption, the image of $f$ is an affine subspace
of dimension $k\leq d$ in $\R^{d+1}$ which is not a subset of
$H_{B,\varrho}$, as a consequence, their intersection is (empty or) a
hyperplane $H_{C,\varrho}$ of dimension $d'<k$, and the
$d'$-dimensional volume $V_S$ of $S_{A,\varrho}:=S_A\cap
H_{C,\varrho}$ satisfies
\begin{equation}\label{VS}
V_S\leq 2^{d'}\sqrt{2}^{k-d'}(2R_f)^d\leq 2^{d+1}R_f^d
\end{equation}
simply because this is the maximum $d'$-dimensional volume of the
intersection of a $d'$-hyperplane with a $k$-cube of volume $(2R_f)^k$
(see Ball, ``Volumes of sections of cubes and related problems'',
1989).  Note that we are using here the bound \eqref{alphaioperator}
on the $\alpha_i$'s to bound the volume of the cube of possible
coefficient vectors.

Now we need to bound the change that can occur by allowing small
changes in the root $\varrho$.  Observe that by Cauchy's bound on the
roots of a polynomial, any root $\varrho$ of $\phi$ satisfies
$|\varrho|\leq 1+R_f\leq 2R_f$ for $R_f\geq 1$.

Now from \eqref{etatonu} we have that
\[
  |\beta_{j-1}|\leq R_f+|\varrho||\beta_{j}|,
  \]
  so that induction gives that
  \begin{equation}\label{betajbound}
    |\beta_{d-j}|\leq \frac{2R_f^j-R_f^{j-1}-R_f}{R_f-1}\leq 2R_f^d
  \end{equation}
  for all $j$.

  In particular, if $|\varrho'-\varrho|\leq \ep$, then any point in
  $H_{B,\varrho}$ is at distance at most $\ep\cdot 2 R_f^d$ from
  $H_{B,\varrho'}$ by \eqref{etatonu} and \eqref{betajbound}, and
  likewise any point in $S_{A,\varrho}$ is at distance at most
  $\ep\cdot 2 R_f^d$ from $S_{A,\varrho'}$, implying (recall $V_S$
  from \eqref{VS}) that the $k$-dimensional volume of
  $\bigcup_{\varrho'\in I}S_{A,\varrho'}$ is at most $(2\ep)\cdot V_S$
  if the interval $I$ has width at most $\ep$.

  In particular, if, with probability $p$, the polynomial $\phi$ has a
  root in the interval $I$ (which has width $\ep$), then
\[
p\leq \kappa_A\cdot 2\ep V_S/V_A,
\]
where $\kappa_A=\kappa/\sqrt{\det(A^TA)}$ is the bound on the $k$-dimensional density of joint distribution of the $\alpha_i$'s inherited from the density bound on the distribution of $\mathbf{a}$.

\end{proof}

\thmDispersionTool*
\begin{proof}
  For simplicity, we assume that every function has exactly $K$
  discontinuities.  For each function $\ell_t$, let $\alpha^{(t)} \in
  \reals^K$ denote the vector whose entries are the discontinuity
  locations of $\ell_t$. That is, $\ell_t$ has discontinuities at
  $\alpha^{(t)}_1, \dots, \alpha^{(t)}_K$, but is otherwise
  $L$-Lispchitz. Since the functions $\ell_1, \ell_2, \dots$ are
  independent, the vectors $\alpha^{(1)}, \alpha^{(2)}, \dots$ are
  also independent.

  For any interval $I \subset \reals$, define the function $f_I : \reals^K \to
  \{0,1\}$ by
  \[
    f_I(\alpha) = \begin{cases}
      1 & \text{if for some $i \in [K]$ we have $\alpha_i \in I$} \\
      0 & \text{otherwise},
    \end{cases}
  \]
  where $\alpha = (\alpha_1, \dots, \alpha_K) \in \reals^K$. The sum
  $\sum_{t=1}^T f_I(\alpha^{(t)})$ counts the number of vectors
  $\alpha^{(1)}, \dots, \alpha^{(T)}$ that have a component in the
  interval $I$ or, equivalently, the number of functions $\ell_1,
  \dots, \ell_T$ that are not $L$-Lipschitz on $I$. We will apply
  VC-dimension uniform convergence arguments to the class $\cF = \{f_I
  : \reals^K \to \{0,1\}\,|\, I \subset \reals \text{ is an
  interval}\}$. In particular, we will show that for an independent
  set of vectors $\alpha^{(1)}, \dots, \alpha^{(T)}$, with high
  probability we have that $\frac{1}{T}\sum_{t=1}^T f_I(\alpha^{(t)})$
  is close to $\expect\bigl[ \frac{1}{T} \sum_{t=1}^T
  f_I(\alpha^{(t)}\bigr]$ for all intervals $I$. This uniform
  convergence argument will lead to the desired bounds.

  We begin by bounding the VC-dimension of $\cF$ by $O(\log K)$. The
  key observation is the following connection between $\cF$ and the
  class of indicator functions for intervals: let $S = \{x^{(1)},
  \dots, x^{(n)}\} \subset \reals^K$ be any collection of $n$ vectors
  in $\reals^K$ and let $P = \{x^{(1)}_1, \dots, x^{(1)}_K, \dots,
  x^{(n)}_1, \dots, x^{(n)}_K\} \subset \reals$ denote the set
  containing the union of their combined $nK$ component values. Now
  consider any pair of intervals $I$ and $I'$. If the indicator
  functions for $I$ and $I'$ agree on all the points in $P$ (i.e., the
  intervals contain exactly the same subset of $P$), then we must have
  that $f_I$ and $f_{I'}$ agree on every vector in $S$. This is
  because if $I$ and $I'$ contain exactly the same subset of $P$, then
  for each vector $x^{(i)}$, both intervals contain the same subset of
  its component values. In particular, either they both contain none
  of the components, or they both contain at least one. In either
  case, we have that $f_I(x^{(i)}) = f_{I'}(x^{(i)})$. This shows that
  the number of distinct ways that functions from the class $\cF$ can
  label the set of vectors $S$ is at most the number of ways that
  indicator functions for intervals can label the set of points $P$.

  Now suppose that the VC-dimension of $\cF$ is $V$. Then there exists
  a set $S \subset \reals^K$ of vectors of size $V$ that is shattered
  by $\cF$. Let $P \subset \reals$ be the set containing the union of
  their combined $VK$ components (as above). From Sauer's Lemma
  together with the fact that the VC-dimension of intervals is $2$, we
  are guaranteed that indicator functions for intervals can label the
  set $P$ of points in at most $(eVK)^2$ distinct ways. By the above
  reasoning, it follows that $\cF$ can label the set $S$ of vectors in
  at most $(eVK)^2$ distinct ways. On the other hand, since $\cF$
  shatters $S$, we know that it can label $S$ in all $2^V$ possible
  ways, and it follows that $2^V \leq (eVK)^2$. Taking logs on both
  sides and rearranging, we have $V \leq \frac{2}{\ln 2} \ln(V) +
  \frac{2\ln(eK)}{\ln 2}$. Using the fact that for any $a \geq 1$ and
  $b \geq 0$, the inequality $y \leq a \ln(y) + b$ implies that $y
  \leq 4a \ln(2a) + 2b$, we further have that $V \leq \frac{8
  \ln(4/\ln 2)}{\ln 2} + \frac{4 \ln(eK)}{\ln 2} = O(\log K)$, as
  required.

  Applying VC-dimension uniform convergence arguments for the class
  $\cF$, for any failure probability $\delta > 0$, if $x^{(1)}, \dots,
  x^{(T)} \in \reals^K$ are independent random vectors (but not
  necessarily identically distributed), then following holds with
  probability at least $1-\delta$ simultaneously for all $f_I \in
  \cF$:
  \begin{align*}
    &\left|
    \frac{1}{T} \sum_{t=1}^T \ell_I(x^{(t)})
    - \expect\left[ \frac{1}{T} \sum_{t=1}^T \ell_I(x^{(t)}) \right]
    \right| \\
    &\leq O\biggl(\sqrt{\frac{\operatorname{VCDim}(\cF) + \log(1/\delta)}{T}}\biggr)
    = O\biggl(\sqrt{\frac{\log(K/\delta)}{T}}\biggr).
  \end{align*}
  In particular, for any point $\rho$ and any radius $\epsilon$, we
  have that $D(T, \epsilon, \rho) = \sum_{t=1}^T f_I(\alpha^{(t)})$,
  where $I = [\rho - \epsilon, \rho + \epsilon]$. Therefore, uniform
  convergence for $\cF$ implies that for all $T \in \mathbb{N}$ and
  all $\epsilon > 0$, and any failure probability $\delta > 0$, we
  have that with probability at least $1-\delta$ the following holds
  for all $\rho \in \reals$:
  \[
    \biggl|
    \frac{1}{T} D(T,\epsilon,\rho)
    -
    \expect\left[\frac{1}{T}  D(T,\epsilon,\rho)\right]
    \biggr|
    \leq O\biggl(\sqrt{\frac{\log(K/\delta)}{T}}\biggr).
  \]
  Multiplying both sides by $T$ and rearranging gives
  \[
    D(T,\epsilon,\rho)
    \leq
    \expect[D(T, \epsilon, \rho)] + O(\sqrt{T \log(K/\delta)}).
  \]
  Taking the maximum of both sides over $\rho \in \reals$, we have
  \[
    \max_{\rho \in \reals} D(T,\epsilon,\rho)
    \leq
    \max_{\rho \in \reals} \expect[D(T, \epsilon, \rho)] + O(\sqrt{T \log(K/\delta)}).
  \]
  This is a high probability bound on the maximum number of
  non-Lipschitz functions among $\ell_1, \dots, \ell_T$ for any
  interval of radius $\epsilon$.  All that remains is to convert this
  into a bound in expectation. Let $\delta = 1/\sqrt{T}$ and let $G$
  denote the high-probability uniform convergence event above. Then we
  have
  \begin{align*}
    \expect&[\max_{\rho \in \reals} D(T,\epsilon,\rho)] \\
    &= \expect[\max_{\rho \in \reals} D(T,\epsilon,\rho) \,|\, G] \prob(G)
     \\& \qquad+\expect[\max_{\rho \in \reals} D(T,\epsilon,\rho) \,|\, \overline{G}] \prob(\overline{G}) \\
    &\leq \max_{\rho \in \reals} \expect[D(T, \epsilon, \rho)] + O(\sqrt{T \log(TK)}) + \sqrt{T} \\
    &= \max_{\rho \in \reals} \expect[D(T, \epsilon, \rho)] + O(\sqrt{T \log(TK)}),
  \end{align*}
  where the last inequality uses the facts that $\prob(G) \leq 1$ and
  $\expect[\max_{\rho \in \reals} D(T,\epsilon,\rho) \,|\,
  \overline{G}] \prob(\overline{G}) \leq T \delta = \sqrt{T}$. This
  argument holds for all $T$ and $\epsilon$, proving the claim.
\end{proof}

Next, we prove a weaker bound that follows from the analysis
techniques of \citet{Balcan17:Dispersion}.

\begin{lemma} \label{lem:dispersionUnionBound}
  Let $\ell_1, \ell_2, \dots : \reals \to \reals$ be independent
  piecewise $L$-Lipschitz functions, each having at most $K$
  discontinuities. Let $D(T, \epsilon, \rho) = \scount{1 \leq t \leq T
  \,|\, \hbox{$\ell_t$ is not $L$-Lipschitz on $[\rho-\epsilon,
  \rho+\epsilon]$}}$ be the (random) number of functions in $\ell_1,
  \dots, \ell_T$ that are not $L$-Lipschitz on the ball
  $[\rho-\epsilon, \rho+\epsilon]$. Moreover, let $\tilde{D}(T,
  \epsilon, \rho) = \scount{(t,i) \in [T] \times [K] \,|\,
  \alpha^{(t)}_i \in [\rho - \epsilon, \rho + \epsilon]}$, where
  $\alpha^{(t)} \in \reals^K$ is the vector of discontinuities of the
  loss $\ell_t$. That is, $\tilde{D}(T, \epsilon, \rho)$ is the number
  of discontinuities of the functions $\ell_1, \dots, \ell_T$ in the
  ball $[\rho - \epsilon, \rho + \epsilon]$. Then we have
  \[
  \expect[\max_{\rho \in \reals} D(T,\epsilon,\rho)]
  \leq
  \max_{\rho \in \reals}\expect[\tilde{D}(T, \epsilon, \rho)] + K\sqrt{T \log(TK)}.
  \]
\end{lemma}

Note that, using the notation of \Cref{lem:dispersionUnionBound}, we
always have $D(T, \epsilon, \rho) \leq \tilde{D}(T,\epsilon,\rho) \leq
KD(T, \epsilon, \rho)$. It follows that
\Cref{lem:dispersionUnionBound} is looser than
\Cref{thm:dispersionTool} in two ways: first, the error term is a
factor $K$ larger. Second, the upper bound of
\Cref{lem:dispersionUnionBound} multiply-counts functions that have
repeated discontinuities in the same ball, while our sharper bound
does not.

\begin{proof}
  For simplicity, we assume that every function $\ell_t$ has exactly
  $K$ discontinuities. The proof techniques can be generalized to the
  case where each functions has at most $K$ discontinuities.

  For each time $t$, let $\alpha^{(t)} \in \reals^K$ be the vector of
  discontinuities of $\ell_t$. That is, $\ell_t$ has discontinuities
  at the points $\alpha^{(t)}_1, \dots, \alpha^{(t)}_K$ and is
  otherwise $L$-Lipschitz.  The key challenge is that the
  discontinuity locations $\alpha^{(t)}_1, \dots, \alpha^{(t)}_K$ are
  not independent.

  Fix any discontinuity index $i \in [K]$ and define $\tilde{D}_i(T,
  \epsilon, \rho) = \scount{1 \leq t \leq T \,|\, \alpha^{(t)}_i \in
  [\rho - \epsilon, \rho + \epsilon]}$. That is, $\tilde D_i(T,
  \epsilon, \rho)$ counts the number of times $t$ for which the
  $i^{\rm th}$ discontinuity $\alpha^{(t)}_i$ of $\ell_t$ lands in the
  interval of radius $\epsilon$ centered on $\rho$. Then we have that
  $\tilde{D}(T, \epsilon, \rho) = \sum_i \tilde{D}_i(T, \epsilon,
  \rho)$ counts the total number of discontinuities that belong to the
  interval of radius $\epsilon$ centered on $\rho$. Since the function
  $\ell_t$ is not $L$-Lipschitz on an interval $I$ only when $I$
  contains some discontinuity for $\ell_t$, we have
  \[
    D(T, \epsilon, \rho)
    \leq \sum_{i=1}^K \tilde{D}_i(T, \epsilon, \rho)
    = \tilde{D}(T, \epsilon, \rho).
  \]

  Next we will apply uniform convergence arguments to obtain high
  probability bounds on each $\tilde{D}_i(T, \epsilon, \rho)$ in terms
  of their expectations. Fix a discontinuity index $i \in [K]$. The
  set of discontinuity locations $\alpha^{(1)}_i, \dots,
  \alpha^{(T)}_i$ are independent and, since intervals have
  VC-dimension 2, applying standard uniform convergence guarantees
  implies that for any $\delta > 0$, with probability at least
  $1-\delta$ the following holds for all $\rho$:
  \[
    \tilde{D}_i(T, \epsilon, \rho) \leq \expect[\tilde{D}_i(T, \epsilon, \rho)] + O(\sqrt{T \log(1/\delta)}).
  \]
  Setting the failure probability to be $1/(K \sqrt{T})$, taking the
  union bound over all $K$ discontinuities, and summing the resulting
  bounds, the following holds with probability at least $1-1/\sqrt{T}$
  for all $\rho$:
  \begin{align*}
    \tilde{D}(T, \epsilon, \rho)
    &= \sum_{i=1}^K \tilde{D}_i(T, \epsilon, \rho) \\
    &\leq \expect\left[\sum_{i=1}^K \tilde{D}_i(T, \epsilon, \rho)\right] + K\cdot O(\sqrt{T \log(KT)}) \\
    &= \expect[\tilde{D}(T, \epsilon, \rho)] + O(K \sqrt{T \log(KT)}).
  \end{align*}
  Using the fact that $D(T, \epsilon, \rho) \leq \tilde D(T, \epsilon,
  \rho)$ and taking the supremum over $\rho$, the following holds with
  probability at elast $1-1/\sqrt{T}$.
  \[
    \max_{\rho \in \reals} D(T, \epsilon, \rho) \leq \max_{\rho \in \reals} \expect[\tilde{D}(T, \epsilon, \rho)] + O(K \sqrt{T \log(KT)}).
  \]
  Let $G$ denote the high-probility uniform convergence event above. Then we have
  \begin{align*}
    &\expect[\max_{\rho \in \reals} D(T,\epsilon,\rho)] \\
    &= \expect[\max_{\rho \in \reals} D(T,\epsilon,\rho) \,|\, G] \prob(G)
     \\&\qquad+\expect[\max_{\rho \in \reals} D(T,\epsilon,\rho) \,|\, \overline{G}] \prob(\overline{G}) \\
    &\leq \max_{\rho \in \reals} \expect[\tilde{D}(T, \epsilon, \rho)] + O(K\sqrt{T \log(TK)}) + \sqrt{T} \\
    &= \max_{\rho \in \reals} \expect[\tilde{D}(T, \epsilon, \rho)] + O(K\sqrt{T \log(TK)}),
  \end{align*}
  as required.
\end{proof}

\thmCurves*
\begin{proof}
  Let the $x$-\emph{extreme} points for $f$ be the points
  $p_0=(x_0,y_0)$ such that there is an open neighborhood $N$ around
  $p_0$ for which $p_0$ has the smallest or largest $x$-coordinate
  among all points $p\in N$ on $f$.  We begin by noting the simple
  fact that $f$ has a bounded number of $x$-extreme points, as a
  consequence of Bezout's theorem bounding the number of intersection
  points of algebraic curves (by the product of the degrees).

  To recall how this goes, let the $y$-\emph{critical} points for $f$
  be the points $p$ at which $\frac d {dy} f(p)=0.$ By the Implicit
  Function Theorem, if $p$ is not $y$-critical, then $f$ locally
  defines a curve whose $y$ coordinate $\gamma(x)$ is a continous
  function of $x$; in particular, we see that if $p$ is $x$-extreme
  then it must be $y$-critical.  On the other hand, there are at most
  $d(d-1)$ $y$-critical points.  This is a simple consequence of
  Bezout's theorem, since if $p_0$ is a $y$-critical point for $f$
  then we have that \[ f(p_0)=0=\frac d {dy} f(p_0), \] where $f$ and
  $\frac d {dy} f$ are polynomials of degrees at most $d$ and $d-1$,
  respectively.

  Now $y$-extreme points can be defined analogously, and there are
  also at most $d(d-1)$ of those.

  Consider now a collection $\ccc$ of $k$ algebraic curves of degree
  at most $d$, and let the set $\ppp$ be the set of all points $p$
  where either:
  \begin{itemize}
  \item Two curves in $\ccc$ intersect at $p$, or
  \item $p$ is $x$-extreme or $y$-extreme for some curve in $\ccc$.
  \end{itemize}
  By Bezout's theorem, there are at most $\binom{k}{2}d^2$ points of
  the first type in $\ccc$, and at most $2kd(d-1)$ points of the
  second type.

  Moreover, consider the horizontal lines $L_c$ defined by the
  equations $y=c$ for constants $c$.  Define the equivalence relation
  $L_{c_1}\sim L_{c_2}$ if the same curves in $\ccc$ intersect
  $L_{c_1}$ and $L_{c_2}$, and in the same order (including with
  multiplicities).  Note that if no points in $\ppp$ lie between the
  lines $L_{c_1}$ and $L_{c_2}$, then $L_{c_1}\sim L_{c_2}$.  In
  particular, there are at most \[ |\ppp|+1\leq \binom{k}{2}d^2
  +2kd(d-1)+1 \] equivalence classes for this equivalence relation.

  Say a subset of $\ccc$ is \emph{hit} by a line segment if the subset
  is exactly the set of curves in $\ccc$ which intersect the segment,
  and \emph{hit} by a line if some segment of the line hits the
  subset.

  Now for a given horizontal line $L_c$, the set of subsets of $\ccc$
  which can hit by any segment of $L_c$ is just determined by the
  pattern intersection points of curves with $L_c$.  The number of
  intersection points is at most $d+1$ by Bezout's theorem, so for a
  given line there are at most $\binom{d+2}{2}$ subsets which can be
  hit.  On the other hand, any two lines in the same equivalence class
  hit precisely the same subsets.  In total, the number of subsets hit
  by any $x$-axis aligned segment is thus \[ \leq
  \binom{d+2}{2}\left(\binom{k}{2}d^2 +2kd(d-1)+1 \right ), \] and the
  number of subsets hit by either an $x$- or $y$-axis aligned segment
  is at most twice this.  We are done since this is less than $2^k$
  for large $k$.
\end{proof}

In the particular case $d=2$, the upper bound in the last line is
\[
24\cdot \left(\binom{k}{2}+k+\frac 1 4\right),
\]
which is less than $2^k$ for $k\geq 11$, showing that $K_2\leq 11.$

\thmCurvesTool*
\begin{proof}
  The key steps of the proof are identical to
  \Cref{thm:dispersionTool}. The main difference is that instead of
  relating the number of ways intervals can label vectors of
  discontinuity points to the VC-dimension of intervals, we instead
  relate the number of ways line segments can label vectors of $K$
  algebraic curves of dimension $d$ to the VC-dimension of line
  segments (when labeling algebraic curves), which from
  \Cref{thm:curves} is constant.
\end{proof}

\section{Appendix for Applications (\Cref{sec:applications})}
\label{app:applications}
\lemKnapsackFeedack*
\begin{proof}
  Sorting the items in step 1 requires $O(n \log n)$ time, while all remaining
  steps make linear passes through the items, resulting in a total running time
  of $O(n \log n)$.

  Computing the permutation $\pi$ in step 1 of \Cref{alg:knapsack} takes $O(n
  \log n)$ time. Finding the item set $S$ in steps 2 and 3 only requires a
  linear pass through the items. Similarly, finding the interval $A$ in steps 4
  and 5 also only requires a linear pass through the items. Therefore, the total
  running time is $O(n \log n)$.

  \citet{Gupta17:PAC} show that for any knapsack instance, the algorithm's
  output is a piecewise constant function of the parameter $\rho$ with at most
  $O(n^2)$ discontinuities. In particular, for each pair of items $i$ and $j$,
  there is a critical parameter value $c_{ij} = \log(v_i/v_j) / \log(s_i/s_j)$
  such that the relative order of items $i$ and $j$ only changes at $\rho =
  c_{ij}$. These critical parameter values partition $\configs$ into $M =
  O(n^2)$ sets $A_1, \dots, A_M$ such that the item ordering is fixed for all
  $\rho \in A_i$. \Cref{alg:knapsack} computes the critical values for each
  consecutive pair of items $\pi(i)$ and $\pi(i+1)$ and outputs the largest
  interval $A$ containing $\rho$ and none of these critical values. For all
  $\rho' \in A$, we must have $\sigma_{\rho'}(\pi(i)) \geq
  \sigma_{\rho'}(\pi(i+1))$ for $i = 1, \dots, n-1$, and therefore the item
  ordering is constant for $\rho' \in A$. It follows that that $A$ does not
  contain $c_{ij}$ for any pair of items $i$ and $j$. On the other hand, the end
  points of $A$ are critical values, so $A$ must be equal to one of the $M$ sets
  $A_i$.
\end{proof}

\lemLinkageDispersion*
\begin{proof}
    The key insight of \citet{Balcan17:Learning} for this family of algorithms is
    that for a fixed distance matrix $D$, the function $\rho \mapsto
    \cA_\rho(D)$ is piecewise constant with at most $O(n^8)$ pieces. That is,
    the algorithm will only output at most $O(n^8)$ different cluster trees, and
    each is produced for some subinterval of the parameter space. Their argument
    is as follows: for any pair of candidate cluster merges, say merging
    clusters $C_1$ and $C_2$ versus $C'_1$ and $C'_2$, we can determine the
    values of the parameter $\rho \in [0,1]$ for which the algorithm would
    prefer to merge $(C_1, C_2)$ instead of merging $(C'_1, C'_2)$ (i.e., the
    values of $\rho$ so that the $\drho$ distance between $C_1$ and $C_2$ is
    smaller than between $C'_1$ and $C'_2$). In particular, the algorithm will
    merge clusters $C_1$ and $C_2$ instead of $C'_1$ and $C'_2$ if $\drho(C_1,
    C_2) \leq \drho(C'_1, C'_2)$ or, equivalently, when
    $
      (1- \rho) \dmin(C_1, C_2) + \rho \dmax(C_1, C_2) \leq (1-\rho) \dmin(C'_1, C'_2) + \rho \dmax(C'_1, C'_2).
    $
    Since the above inequality is linear in $\rho$, there is a single critical
    value of the parameter, given by
    $
      c = \frac{\dmin(C'_1, C'_2) - \dmin(C_1, C_2)}{\dmax(C_1, C_2) - \dmin(C_1, C_2) + \dmin(C'_1, C'_2) - \dmax(C'_1, C'_2)}
    $
    such that the relative preference of merging $C_1$ and $C_2$ or $C'_1$ and
    $C'_2$ changes only at $\rho = c$. Moreover, the definition of $c$ only
    depends on a collection of $8$ points: the closest and farthest pair between
    $C_1$ and $C_2$ and between $C'_1$ and $C'_2$. In particular, every such
    critical parameter value $c$ is given by
    \begin{equation}
    c = \frac{d^{(t)}_{rr'} - d^{(t)}_{ii'}}{d^{(t)}_{jj'} - d^{(t)}_{ii'} + d^{(t)}_{rr'} -
    d^{(t)}_{ss'}}
    \label{eq:linkageBoundary}
    \end{equation}
    where $i, i', j, j', r, r', s, s' \in [n]$ are the indices of $8$ points.
    Similarly to the knapsack example, we show that each critical parameter
    value is random and has a density function bounded by $16(\kappa B)^2$. From
    this, it follows that for any interval $I$ of radius $\epsilon$, the
    expected total number of critical values summing over all instances $t = 1,
    \dots, T$ that land in interval $I$ is at most $32 T\epsilon (\kappa B)^2$.
    This also bounds the expected number of functions $\ell_1, \dots, \ell_T$
    that are not constant on $I$. By \Cref{thm:dispersionTool}, the functions
    are $f$-dispersed for $f(T,\epsilon) = 32 T \epsilon (\kappa B)^2 + \sqrt{T
    \log(Tn)} = \tilde O(T\epsilon + \sqrt{T})$, also implying
    $\frac{1}{2}$-dispersion.

    When the four distances present in the equation for $c$ are distinct entries
    of the distance matrix $D$, then they are independent. However, it is
    possible that the closest and furthest pair of points between a pair of
    clusters can be the same, for example, when both clusters consist of just a
    single point. In this case, the corresponding distances are no longer
    independent, and we will need to modify our analysis slightly. Note a
    critical parameter $c$ only arises for competing pairs of merges $(C_1,
    C_2)$ and $(C'_1, C'_2)$ that differ on at least one cluster (since
    otherwise both merges are identical). Moreover, since the set of clusters at
    any given round of the algorithm partition the data, any pair of clusters
    the algorithm encounters are either equal or disjoint. From this it follows
    that there are only four cases to consider depending on whether the closest
    and farthest pairs of points between $C_1$ and $C_2$ are the same, and
    whether the closest and farthest pairs of points between $C'_1$ and $C'_2$
    are the same. That is, whether $(i,i') = (j,j')$ and whether $(s,s') =
    (r,r')$.

    \vspace{1em}\noindent\textit{Case 1: $(i,i') \neq (j,j')$ and $(r,r') \neq
    (s,s')$.} Let $X = d_{rr'} - d_{ii'}$ and $Y = d_{jj'} - d_{ss'}$. Rewriting
    expression for $c$ given in \eqref{eq:linkageBoundary}, we have that $c = X
    / (X + Y)$. Moreover, both $X$ and $Y$ are the sum of two independent random
    variables having $\kappa$-bounded densities, so from \Cref{lem:kappaSum}, it
    follows that $X$ and $Y$ also have densities bounded by $\kappa$. Next,
    since $X$ and $Y$ are independent, take values in $[-2M, 2M]$, and have
    $\kappa$-bounded densities, \Cref{lem:kappaRatio1} ensures that the ratio $X
    / (X + Y)$ has an $16 (\kappa M)^2$ bounded density.

    \vspace{1em}\noindent\textit{Case 2: $(i,i') = (j,j')$ and $(r,r') \neq
    (s,s')$.} In this case, we are guaranteed that $d_{ii'} = d_{jj'}$, and the
    expression for $c$ simplifies to
    \[
      c = \frac{d_{rr'} - d_{ii'}}{d_{rr'} - d_{ss'}}
    \]
    Defining $X = -d_{ii'}$, $Y = -d_{ss'}$, and $Z = d_{rr'}$, we have that
    $\beta = (X+Z)/(Y+Z)$. The variables $X$, $Y$, and $Z$ are independent, each
    have $\kappa$-bounded densities, and $|Y| \leq M$ and $|Z| \leq M$ with
    probability 1. Applying \Cref{lem:kappaRatio2} to these random variables
    guarantees that the density function for $\beta$ is $4(\kappa M)^2$-bounded.

    \vspace{1em}\noindent\textit{Case 3: $(i,i') \neq (j,j')$ and $(r,r') =
    (s,s')$.} This case is symmetric to case 2 and an identical argument
    applies.

    \vspace{1em}\noindent\textit{Case 4: $(i,i') = (j,j')$ and $(r,r') =
    (s,s')$.} In this case, the $\drho$ distance between $C_1$ and $C_2$ is
    constant, as is the $\drho$ distance between $C'_1$ and $C'_2$. Therefore,
    for all values of $\rho$ we will prefer to merge the same pair of clusters
    and there is no critical parameter value where we switch from one merge to
    the other.

    In every case, the density of the critical parameter value $\beta$ is upper
    bounded by $16\kappa^2 M^2$, completing the proof.
\end{proof}

\lemRhoLinkageBookkeeping*
\begin{proof}
  The algorithm performs $n-1$ merges. For each merge, the algorithm makes two
  passes through the $O(n^2)$ clusters in order to find the closest pair, as
  well as to update the interval $(\rmin, \rmax)$. These passes both require us
  to compute the $\drho$ distance between all pairs of clusters. However,
  starting from the input matrix $D$, we can maintain two distance matrix
  $D^\text{min}$ and $D^\text{max}$ storing the minimum and maximum distances
  between the current set of clusters, respectively. After merging two clusters,
  these distance matrices can be updated in $O(n)$ time, since at most $O(n)$
  distances change. It follows that finding the closest pair of clusters and
  updating the interval $(\rmin, \rmax)$ can be done in $O(n^2)$ time per merge.
  This leads to a total running time of $O(n^3)$.

  \citet{Balcan17:Learning} prove that there exists a partition $A_1, \dots,
  A_M$ of $\configs$ into $M = O(n^8)$ intervals such that the algorithm output
  is constant for $\rho \in A_i$. In particular, for any pair of possible
  cluster merges $(C_1,C_2)$ and $(C'_1, C'_2)$ with $\dmin(C_1, C_2) <
  \dmin(C'_1, C'_2)$, the algorithm prefers to merge $C_1$ and $C_2$ over $C'_1$
  and $C'_2$ for all values of the parameter $\rho < c(C_1, C_2, C'_1, C'_2)$.
  Moreover, since $c(C_1, C_2, C'_1, C'_2)$ only depends on 8 points---the
  closest and farthest pairs of points between $C_1$ and $C_2$ and between
  $C'_1$ and $C'_2$---and there are only $O(n^8)$ ways to select $8$ points,
  these critical parameter values partition $\configs$ into the $M = O(n^8)$
  intervals. For $\rho \in A_i$, the ordering on all possible merges is fixed,
  so the algorithm will output the same cluster tree.

  Finally, on each iteration of the algorithm, we iterate through all $O(n^2)$
  pairs of clusters $(C_1', C_2')$ that the algorithm did not merge. For each,
  we calculate the critical parameter value $c(C_1, C_2, C'_1, C'_2)$, which is
  the value of $\rho$ at which the algorithm would prefer to merge $(C'_1,
  C'_2)$ over $(C_1, C_2)$. We shrink the interval $(\rmin, \rmax)$ so that it
  does not contain any of these critical values. It follows that the interval
  $(\rmin, \rmax)$ satisfies the following invariant: for all $\rho' \in (\rmin,
  \rmax)$, the sequence of cluster merges made by the algorithm with parameter
  $\rho'$ up until the current iteration would match those made by the algorithm
  with parameter $\rho$. In particular, when the algorithm returns, we are
  guaranteed that the same cluster tree would be output for all parameter values
  $\rho' \in (\rmin, \rmax)$. Since the endpoints $\rmin$ and $\rmax$ always
  belong to the $M = O(n^8)$ critical parameter values, there are at most $M =
  O(n^8)$ intervals the algorithm might output for a fixed clustering instance.
\end{proof}

\subsection{Single Parameter Piecewise Unique Algorithms.} \label{app:binSearch}

Next we provide a general approach for obtaining semi-bandit feedback that
applies to many single-parameter algorithms. This enables semi-bandit feedback,
but we still rely on problem-specific dispersion analysis. This approach applies
to any algorithm with a single real-valued parameter whose output is both a
piecewise constant function of the parameter for any instance, and such that no
output value is repeated across any distinct intervals in the piecewise
decomposition. We call such an algorithm \emph{single-parameter
piecewise-unique}. Without loss of generality, we assume that the parameter
space is given by $\configs = [0,1]$. Let  $\cA : \Pi \times [0,1] \to \cY$ be
an algorithm mapping problem instances $x \in \Pi$ and parameters $\rho \in
[0,1]$ to outputs in some space $\cY$. Given a parameter $\rho \in [0,1]$ and a
problem instance $x$, and an accuracy parameter $\epsilon > 0$, we will return
both $\cA(x, \rho)$, together with an interval $I = [\rmin, \rmax]$ such that
for all $\rho' \in I$ we have $\cA(x, \rho') = \cA(x, \rho)$. Moreover, for any
point $\rho' \not \in [\rmin - \epsilon, \rmin + \epsilon]$, we have $\cA(x,
\rho') \neq \cA(x, \rho)$. In other words, the interval $I$ output by the
algorithm is nearly the largest piecewise constant interval containing $\rho$.
The high level idea of our approach is to run binary search twice to determine
the upper and lower bounds $\rmax$ and $\rmin$, respectively. Each search will
require that we run the algorithm $\cA$ at most $O(\log 1/\epsilon)$ times. In
cases where the algorithm parameters are specified using $b$ bits of precision,
then this procedure exactly determines the interval using $O(b)$ invocations of
the base algorithm. Pseudocode is given in \Cref{alg:blackbox}. Steps 3 and 4
perform binary search to find the upper bound on the constant interval, while
steps 5 and 6 perform binary search to find the lower bound.

\begin{lemma}
  Let $\cA : \Pi \times [0,1] \to \cY$ be any single-parameter piecewise-unique
  algorithm and suppose $y_\rho$ and $I = [\rmin, \rmax]$ is output by
  \Cref{alg:blackbox} when run on $\cA$ with problem instance $x \in \Pi$,
  parameter $\rho \in [0,1]$, and target accuracy $\epsilon$. Then
  \Cref{alg:blackbox} runs the base algorithm $\cA$ at most $O(\log 1/\epsilon)$
  times and we have that $\cA(x, \rho') = y_\rho$ for all $\rho' \in I$, $\rho
  \in I$, and for all $\rho' \not \in [\rmin - \epsilon, \rmax + \epsilon]$ we
  have $\cA(x, \rho') \neq y_\rho$.
\end{lemma}
\begin{proof}
  From step 1 of the algorithm, we know that $\cA(x, \rho) = y_\rho$, by
  definition. Since the algorithm is single-parameter and piecewise-unique, we
  know that $\cA(x, \rho')$ will output $y_\rho$ for all $\rho'$ belonging to
  some interval $[\rmin^*, \rmax^*]$ containing $\rho$, and it will not output
  $y_\rho$ for any point outside that interval. In particular, restricted to the
  interval $[\rho, 1]$, there is exactly one critical parameter value, namely
  $\rmax^*$ below which the algorithm always outputs $y_\rho$ and above which
  the algorithm always outputs something different. The binary search performed
  in step 3 guarantees that $\rmax^*$ is always contained in the interval
  $[a,b]$, yet on each iteration the length of the interval is halved.
  Similarly, each iteration of the binary search in step 6 guarantees that
  $\rmin^* \in [c,d]$, and the width of the interval halves on each iteration.
  Each iteration of both binary search instances requires us to run the base
  algorithm $\cA$ once, and we will require $O(\log 1/\epsilon)$ iterations to
  guarantee the width of both intervals is less than $\epsilon$.

  Since $a \leq \rmax^*$ and $d \geq \rmin^*$, we have $[a,d] \subset [\rmin^*,
  \rmax^*]$ and it follows that $\cA(x, \rho') = y_\rho$ for all $\rho' \in
  [a,d]$, as required. Moreover, we know that $a + \epsilon \geq b \geq \rmax^*$
  and $d - \epsilon \leq c \leq \rmin^*$, implying that $\cA(x, \rho') \neq
  y_\rho$ for all $\rho' \not \in [a-\epsilon, d+\epsilon]$, as required.
  \Cref{fig:blackbox} depicts the relation between $[a,b]$, $[c,d]$, and
  $[\rmin^*, \rmax^*]$ at the end of the algorithm.
\end{proof}

\begin{algorithm}
\noindent \textbf{Input:} Algorithm $\cA : \Pi \times [0,1] \to \cY$ , parameter $\rho \in [0,1]$, problem instance $x \in \Pi$.
\begin{enumerate}[nosep, leftmargin=*]
\item Let $y_\rho \leftarrow \cA(x, \rho)$ be the output of $\cA$ run on $x$ with parameter $\rho$.
\item Let $a \leftarrow 0$ and $b \leftarrow \rho$.
\item While $b - a > \epsilon$:
  \begin{enumerate}[nosep, leftmargin=*]
  \item Set $m \leftarrow (a + b)/2$.
  \item If $\cA(x, m) = y_\rho$ then set $b \leftarrow m$
  \item Otherwise set $a \leftarrow m$.
  \end{enumerate}
\item Let $\rmin \leftarrow b$.
\item Let $c \leftarrow \rho$ and $d \leftarrow 1$.
\item While $d - c > \epsilon$:
  \begin{enumerate}[nosep, leftmargin=*]
  \item Set $m \leftarrow (c + d)/2$.
  \item If $\cA(x, m) = y_\rho$ then set $c \leftarrow m$
  \item Otherwise set $d \leftarrow m$.
  \end{enumerate}
\item Let $\rmin \leftarrow c$
\item Output $y_\rho$ and interval $I = (\rmin, \rmax)$.
\end{enumerate}
\caption{Blackbox Bandit Feedback for Single-parameter Algorithms}
\label{alg:blackbox}
\end{algorithm}

\begin{figure}
  \centering
  \includegraphics[width=0.3\textwidth]{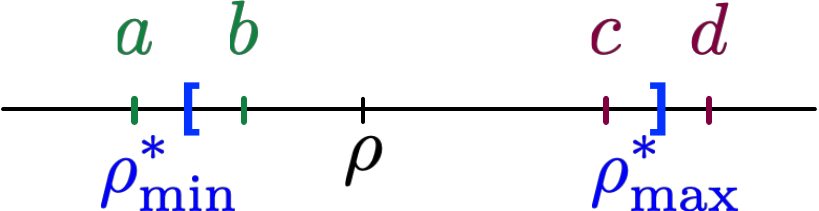}
  \caption{Relationship between the binary search intervals $[a,b]$ and $[c,d]$ and the true interval $[\rmin^*, \rmax^*]$ on which $\cA(x, \rho')$ outputs $y_\rho$.}
  \label{fig:blackbox}
\end{figure}

\subsection{Learning a Metric and Interpolating Between Single and Complete Linkage}
\label{sec:metricAndLinkage}

Finally, we consider an extension of the clustering algorithm family introduced
in \Cref{sec:clusteringMerge} that allows us to also learn the best metric to
use for a specific clustering domain. We suppose that the clustering domain is
equipped with two base metrics and our goal is to learn the best convex
combination of them. For example, one metric could be based on a neural network
feature embedding and the other could be a hand-designed metric based on domain
expertise. Formally, each clustering instance is described by a pair of distance
matrices $D_0$ and $D_1$ specifying the distances between each pair of points
according to the two metrics.

The algorithm family we consider called $\rho$-metric-linkage has two
parameters, $\rho \in [0,1]$ which is used to interpolate between single and
complete linkage as in \Cref{sec:clusteringMerge}, and $\alpha \in [0,1]$, which
is used to interpolate between the two distance metrics. In particular, for any
$\alpha \in [0,1]$, let $D_\alpha = (1-\alpha) D_0 + \alpha D_1$ and denote its
entries by $(d_{\alpha, i, j})$. The algorithm with parameters $\rho$ and
$\alpha$ repeatedly merges the pair of clusters $A$ and $B$ that are closest
according to the distance $d_{\rho,\alpha}(A,B) = (1-\rho) \min_{a \in A, b \in
B} d_{\alpha,a,b} + \rho \max_{a \in A, b \in B} d_{\alpha,a,b}$.

We provide dispersion analysis for this algorithm family under the assumption
that both distance matrices for each problem instance are $\kappa$-smooth,
meaning that their entries are independent and each have $\kappa$-bounded
discontinuities.

\begin{theorem}
  Consider an adversary choosing clustering instances where the $t^{\rm th}$
  instance has two symmetric distance matrices $D^{(t)}_0, D^{(t)}_1 \in
  [0,B]^{n \times n}$ and for all $i < j$, the $(i,j)^\text{th}$ entry of
  $D^{(t)}_0$ and $D^{(t)}_1$ are $\kappa$-smooth. The loss functions $\ell_1,
  \ell_2, \dots$ defined above are piecweise constant and $f$-dispersed for
  $f(T,\epsilon) = O(T n^8 \kappa^3 \epsilon + \sqrt{T \log(Tn)})$ and
  $\beta$-dispersed for $\beta = 1/2$.
\end{theorem}
\begin{proof}[Proof sketch.]
  First, we argue that for any pair of distance matrices $D_0$ and $D_1$, we can
  partition the $\alpha$-parameter space into $O(n^4)$ regions such that on each
  region the ordering over pairs of points $(i,j)$ according to the distance
  $d_{\alpha, i, j}$ is fixed. In particular, on each region, the closest and
  farthest pair of points between any pair of clusters are constant. Next,
  restricted to each of these regions, we find that we can express the distance
  $d_{\rho,\alpha}(A,B)$ between any pair of clusters as a quadratic function of
  the parameters $\rho$ and $\alpha$. This implies that for any pair of
  candidate merges $(A, B)$ and $(A', B')$, the algorithm prefers to merge
  $(A,B)$ over $(A',B')$ whenever the quadratic polynomial $q_{A,B,A',B'}(\rho,
  \alpha) = d_{\rho,\alpha}(A,B) - d_{\rho,\alpha}(A',B')$ is negative. While
  there are exponentially many choices of the $4$ clusters, the coefficients of
  this quadratic depend on the distances between a set of $8$ points, and
  therefore there are only $O(n^8)$ unique quadratic functions. Next, we argue
  that for any fixed $\alpha$, the polynomial $\rho \mapsto q_{A,B,A',B'}(\rho,
  \alpha)$ has coefficients with joint density bounded by $O(\kappa^3)$ and
  whose magnitude are bounded by $O(B)$ (and similarly for fixing any $\rho$ and
  treating the polynomial as a function of $\alpha$ alone). Now consider any
  pair of parameters $(\rho, \alpha)$ and $(\rho', \alpha')$ that are within
  distance $\epsilon$. The piecewise-linear path given by the line segment from
  $(\rho, \alpha)$ to $(\rho, \alpha')$ followed by the line segment $(\rho,
  \alpha')$ to $(\rho', \alpha')$ has total length $O(\epsilon)$. The above
  arguments combined with \Cref{thm:polyRoots} ensures that this path contains a
  root of the quadratic $q_{A,B,A',B'}$ with probability at most $O(\kappa^3
  \epsilon)$. Therefore, the expected number of discontinuities intersecting the
  rectilinear path between any pair of paremeter vectors at distance $\epsilon$
  is upper bounded by $O(T n^8 \kappa^3 \epsilon)$ (by summing over the $O(n^8)$
  different quadratics obtained from each of the $T$ instances). Applying
  \Cref{thm:curvesTool} we are guaranteed that the worst rectilinear path has at
  most $\tilde O(T n^8 \kappa^3 \epsilon + \sqrt{T})$ discontinuities, which
  proves the claim. 
\end{proof}

\section{Transformations of Bounded Densities}
\label{app:transformations}
In this section we summarize several useful results that provide upper bounds on
the density of random variables that are obtained as functions of other random
variables with bounded density functions. These results allow us to reason about
the distribution of discontinuity locations that arise as transformations of
random problem parameters in algorithm configuration instances.

In many cases, we make use of the following result:

\begin{theorem}[Density Function Change of Variables] \label{thm:cov}
  Let $X \in \reals^d$ be a random vector with joint probability density
  function $f_X : \reals^d \to [0,\infty)$ and let $\phi : \reals^d \to
  \reals^n$ be any bijective differentiable function. Then the random vector $Y
  = \phi(X)$ also has a density function $f_Y : \reals^n \to [0,\infty)$ given
  by $f_Y(y) = |\det(J_{\phi^{-1}}(y))| f_X(\phi^{-1}(y))$, where
  $J_{\phi^{-1}}(y)$ denotes the Jacobian of $\phi^{-1}$ evaluated at $y$.
\end{theorem}

\begin{lemma} [Lemma 6 from \citep{Balcan17:Dispersion}] \label{lem:logRatio}
  Suppose $X$ and $Y$ are random variables taking values in $(0,1]$ and suppose
  that their joint distribution is $\kappa$-bounded. Then the distribution of $Z
  = \ln(X/Y)$ is $\kappa/2$-bounded.
\end{lemma}

\begin{lemma} [Lemma 8 from \citep{Balcan17:Dispersion}] \label{lem:rescale}
  Suppose $X$ is a random variable with a $\kappa$-bounded density and suppose
  $c$ is a constant. Then $Z = X/c$ has a $c\kappa$-bounded density
\end{lemma}

\begin{lemma} \label{lem:kappaSum}
  Let $X$ and $Y$ be two independent random variables each having densities
  upper bounded by $\kappa$. The random variable $U = X + Y$ has density $f_U$
  satisfying $f_U(u) \leq \kappa$ for all $u$.
\end{lemma}
\begin{proof}
  Let $f_X$ and $f_Y$ be the density functions for $X$ and $Y$, respectively. The
  density for $U$ is the convolution of $f_X$ and $f_Y$. With this, we have
  \[
  f_U(u)
  = \int_{-\infty}^\infty f_X(u - y) f_Y(y) \, dy
  \leq \int_{-\infty}^\infty \kappa f_Y(y) \, dy
  = \kappa.
  \]
  It follows that $U = X+Y$ has a density that is upper bounded by $\kappa$.
\end{proof}

\begin{lemma} \label{lem:ratioDensityBound}
  Let $X$ and $Y$ be random variables with joint density $f_{XY}$ that is
  $\kappa$-bounded and such that $|Y| \leq M$ with probability 1 and let $U =
  X/Y$. Then the density function $f_U$ is $\kappa M^2$-bounded.
\end{lemma}
\begin{proof}
  Consider the change of variables given by $U = X/Y$ and $V = Y$. This
  corresponds to the transformation function $\phi(x,y) = (x/y, y)$. The inverse
  of $\phi$ is given by $\phi^{-1}(u,v) = (uv, v)$. The Jacobian of $\phi^{-1}$
  is
  \[
    J_{\phi^{-1}}(u,v) = \left[\begin{matrix}
      v & u \\
      0 & 1
    \end{matrix}\right],
  \]
  whose determinant is always equal to $v$. Therefore, the joint density of $U$
  and $V$ is given by
  \[
    f_{UV}(u, v) = |v| f_{XY}(uv, v).
  \]
  To get the marginal density for $U$, we integrate over $v$ and use the fact
  that the density $f_{XY}(x,y) = 0$ whenever $|y| > M$. This gives
  \[
    f_U(u)
    = \int_{-M}^M |v| f_{XY}(uv, v) \, dv
    \leq \kappa \int_{-M}^M |v| \, dv
    = \kappa M^2.
  \]
  It follows that the density for $U$ satisfies  $f_U(u) \leq \kappa M^2$ for
  all $u$, as required.
\end{proof}

\begin{lemma} \label{lem:kappaRatio1}
  Let $X$ and $Y$ be independent random variables with $\kappa$-bounded
  densities so that $|X| \leq M$ and $|Y| \leq M$ with probability one and
  define $Z = X / (X+Y)$. The random variable $Z$ has a density function $f_Z$
  that is $4 \kappa^2 M^2$-bounded.
\end{lemma}
\begin{proof}
  Consider the change of variables given by $U = X$ and $V = X + Y$. We will
  argue that the joint density $f_{UV}$ is $\kappa^2$-bounded. Then, since
  $|X+Y| \leq 2M$ with probability 1, we can apply \Cref{lem:ratioDensityBound}
  to ensure that the density of $Z = U/V$ is bounded by $\kappa^2(2M)^2 =
  4\kappa^2 M^2$, as required.

  It remains to bound the joint density of $U = X$ and $V = X + Y$. This change
  of variables corresponds to the transformation function $\phi(x,y) = (x,
  x+y)$, whose inverse is given by $\phi^{-1}(u,v) = (u, v - u)$. The Jacobian
  of $\phi^{-1}$ is given by
  \[
    J_{\phi^{-1}}(u,v) = \left[\begin{matrix}
      1 & 0 \\
      -1 & 1 \\
    \end{matrix}\right],
  \]
  whose determinant is always $1$. It follows that the joint density for $(U,V)$
  is given by $f_{UV}(u,v) = f_{XY}(u, v-u) = f_{X}(u) f_Y(v-u) \leq \kappa^2$,
  as required.
\end{proof}

\begin{lemma} \label{lem:kappaRatio2}
  Let $X$, $Y$, and $Z$ be independent random variables with $\kappa$-bounded
  densities such that $|Y| \leq M$, and $|Z| \leq M$ with probability one. Then
  the random variable $R = \frac{X+Y}{Z+Y}$ has a density $f_R$ that satisfies
  $f_R(u) \leq 4\kappa^2M^2$.
\end{lemma}
\begin{proof}
  Consider the change of variables given by $U = X + Y$, $V = Z + Y$. We will
  argue that the joint density $f_{UV}$ for $U$ and $V$ is $\kappa^2$-bounded.
  Then, since $|V| = |Z+Y| \leq 2M$ with probability 1, we can apply
  \Cref{lem:ratioDensityBound} to ensure that the density of $R = U/V$ is
  bounded by $4\kappa^2 M^2$, as required.

  It remains to bound the joint density of $U = X + Y$ and $V = Z + Y$. Consider
  the change of variables given by $U = X + Y$, $V = Z + Y$, and $W = Y$. This
  corresponds to the transformation function $\phi(x,y,z) = (x+y, z+y, y)$, and
  has inverse $\phi^{-1}(u,v,w) = (u - w, w, v - w)$. The Jacobian of $\phi^{-1}$
  is given by
  \[
    J_{\phi^{-1}}(u,v,w) = \left[\begin{matrix}
      1 & 0 & -1 \\
      0 & 0 & 1 \\
      0 & 1 & -1
    \end{matrix}\right],
  \]
  which always has determinant given by $-1$. It follows that the joint density
  for $(U,V,W)$ is given by
  \begin{align*}
    f_{UVW}(u,v,w)
    &= f_{XYZ}(u - w, w, v - w) \\
    &= f_X(u - w) f_Y(w) f_Z(v-w).
  \end{align*}
  To get the joint density over only $U$ and $V$ we integrate over $w$:
  \begin{align*}
  f_{UV}(u,v)
  &= \int_{-\infty}^\infty f_X(u - w) f_Y(w) f_Z(v-w) \, dw \\
  &\leq \kappa^2 \int_{-\infty}^\infty f_Y(w) \, dw \\
  &= \kappa^2,
  \end{align*}
  as required.
\end{proof}

\end{document}